\documentclass[a4paper]{article}

\usepackage{graphicx}
\newcommand\hl[1]{{#1}}

\usepackage{amsmath,amssymb}
\usepackage{amsthm}
\usepackage{bussproofs}
\usepackage{enumerate}
\usepackage{hyperref}
\usepackage{listings}
\usepackage[binary-units=true]{siunitx}
\usepackage[caption=false]{subfig}
\usepackage{pgfplots}
\pgfplotsset{compat=1.3}
\usetikzlibrary{patterns}
\usepgfplotslibrary{external} 

\definecolor{blue_1}{RGB}{0,90,169}
\definecolor{blue_2}{RGB}{0,131,204}
\definecolor{blue_1b}{RGB}{93,133,195}
\definecolor{blue_2b}{RGB}{0,156,218}
\definecolor{gray_1}{RGB}{142,152,178}
\long\def\ignore#1{}
\newcommand{\types}{\mathcal{T}}
\newcommand{\Var}{\mathcal{V}}
\newcommand{\syntrue}{\top}

\newcommand{\fv}{\mathrm{fv}}

\newcommand{\pospol}{\mathfrak{t\!t}}
\newcommand{\negpol}{\mathfrak{f\!f}}
\newcommand{\emptyclause}{\square}
\newcommand{\clause}[1]{\mathcal{#1}}

\newcommand{\calculus}{\text{EP}}
\newcommand{\paramodrule}{(Para)}
\newcommand{\paramodapp}{\ensuremath{\mathrm{Para}}}
\newcommand{\factorrule}{(EqFac)}
\newcommand{\factorapp}{\ensuremath{\mathrm{EqFac}}}
\newcommand{\primsubstrule}{(PS)}
\newcommand{\primsubstapp}{\ensuremath{\mathrm{PS}}}

\newcommand{\cnfapp}{\ensuremath{\mathsf{CNF}}}

\newcommand{\uniapp}{\ensuremath{\mathsf{UNI}}}

\newcommand{\funcposrule}{(PFE)}
\newcommand{\funcposapp}{\ensuremath{\mathrm{PFE}}}
\newcommand{\funcnegrule}{(NFE)}

\newcommand{\boolposrule}{(PBE)}
\newcommand{\boolposapp}{\ensuremath{\mathrm{PBE}}}
\newcommand{\boolnegrule}{(NBE)}

\newcommand{\trivrule}{(Triv)}

\newcommand{\bindrule}{(Bind)}

\newcommand{\decomprule}{(Decomp)}

\newcommand{\flexrigidrule}{(FlexRigid)}

\newcommand{\flexflexrule}{(FlexFlex)}

\newcommand{\sk}{\mathit{sk}}
\newcommand{\FV}{{X}}

\newcommand{\approxbinding}[2]{\ensuremath{\mathcal{GB}_{#1}^{#2}}}

\newtheorem{theorem}{Theorem}
\newtheorem{example}{Example}

\begin{document}

\title{Extensional Higher-Order Paramodulation in Leo-III\thanks{This work has
been supported by the DFG under grant BE 2501/11-1 (Leo-III) and by
the Volkswagenstiftung (''Consistent Rational Argumentation in Politics'').}}


\author{Alexander Steen\and
        Christoph Benzm\"uller
}


\ignore{
\institute{A. Steen \href{https://orcid.org/0000-0001-8781-9462}{[0000-0001-8781-9462]} \at
              University of Luxembourg, FSTC,
              Esch-sur-Alzette, Luxembourg.\\
              \email{alexander.steen@uni.lu}
           \and
           C. Benzm\"uller \href{https://orcid.org/0000-0002-3392-3093}{[0000-0002-3392-3093]} \at
              Freie Universit\"at Berlin, Dep. of Mathematics and Computer Science, Berlin, Germany, and\\
              University of Luxembourg, FSTC, Esch-sur-Alzette, Luxembourg.\\
              \email{c.benzmueller@fu-berlin.de}
}
}


\date{\texttt{alexander.steen@uni.lu}, \texttt{c.benzmueller@fu-berlin.de}}

\maketitle

\begin{abstract}
Leo-III is an automated theorem prover for extensional type theory
with Henkin semantics and choice. Reasoning with primitive equality
is enabled by adapting paramodulation-based proof search to
higher-order logic. The prover may cooperate with multiple external
specialist reasoning systems such as first-order provers and SMT
solvers. Leo-III is compatible with the TPTP/TSTP
framework for input formats, reporting results and proofs,
and standardized communication between reasoning systems,
enabling e.g. proof reconstruction from within proof assistants
such as Isabelle/HOL.

Leo-III supports reasoning in polymorphic first-order and higher-order
logic, in many quantified normal modal logics, as well as in
different deontic logics.
Its development had initiated the ongoing extension of the TPTP
infrastructure to reasoning within non-classical logics.



\end{abstract}


\section{Introduction} \label{intro}

Leo-III is an automated theorem prover (ATP) for classical higher-order logic (HOL)
with Henkin semantics and choice. In contrast to its predecessors,
LEO and LEO-II~\cite{DBLP:conf/cade/BenzmullerK98a,DBLP:journals/jar/BenzmullerSPT15},
that were based on resolution proof search, Leo-III implements a higher-order
paramodulation calculus, which aims at improved performance for equational reasoning~\cite{steen2018}.
In the tradition of the Leo prover family, Leo-III collaborates with external reasoning
systems, in particular, with first-order ATP systems, such as E~\cite{Schulz:2002:EBT:1218615.1218621},
iProver~\cite{DBLP:conf/cade/Korovin08} and Vampire~\cite{DBLP:journals/aicom/RiazanovV02},
and with SMT solvers such as CVC4~\cite{DBLP:conf/cav/BarrettCDHJKRT11}.
Cooperation is not restricted to first-order systems, and further specialized systems such as
higher-order (counter-)model finders may be utilized by Leo-III.

Leo-III accepts all common TPTP dialects~\cite{DBLP:journals/jar/Sutcliffe17}
as well as their recent extensions to polymorphic types~\cite{DBLP:conf/cade/BlanchetteP13,DBLP:conf/cade/KaliszykSR16}.
During the development of Leo-III, careful attention has been paid to providing maximal
compatibility with existing systems and conventions of the peer community, especially
to those of the TPTP infrastructure.
The prover returns results according to the standardized TPTP SZS ontology,
and it additionally produces verifiable TPTP-compatible proof certificates
for each proof that it finds.

The predecessor systems LEO and LEO-II
pioneered the area of cooperative resolution-based theorem proving for Henkin semantics.
LEO (or LEO-I) was designed as an ATP component of the
proof assistant and proof planner $\Omega$MEGA~\cite{DBLP:journals/japll/SiekmannBA06}
and hard-wired to it.
Its successor, LEO-II, is a stand-alone HOL ATP system based on
Resolution by Unification and Extensionality (RUE)~\cite{DBLP:conf/cade/Benzmuller99},
and it supports reasoning with primitive equality.

The most recent incarnation of the Leo prover family, Leo-III,
comes with improved reasoning performance in particular for equational problems, and with a more flexible
and effective architecture for cooperation with external specialists.\footnote{
Note the different capitalization of Leo-III as opposed to LEO-I and LEO-II. This is motivated
by the fact that Leo-III is designed to be a general purpose system rather than a subcomponent
to another system. Hence, the original capitalization derived from the phrase
''\underline{L}ogical \underline{E}ngine for \underline{O}mega'' (LEO) is not continued.
}
Reasoning in higher-order quantified non-classical logics, including many normal modal logics,
and different versions of deontic logic is enabled by an integrated
shallow semantical embedding approach~\cite{DBLP:journals/lu/BenzmullerP13}.
In contrast to other HOL ATP systems, including LEO-II, 
for which it was necessary for the user to manually conduct the tedious and error-prone
encoding procedure before passing it to the prover, Leo-III is the first
ATP system to include a rich library of these embeddings, transparent
to the user~\cite{DBLP:conf/lpar/GleissnerSB17}.
These broad logic competencies make Leo-III, up to the authors' knowledge,
the most widely applicable ATP system
for propositional and quantified, classical and non-classical logics 
available to date.
This work has also stimulated the currently ongoing extension of
the TPTP library to non-classical reasoning.\footnote{
See \url{http://tptp.org/TPTP/Proposals/LogicSpecification.html}.
}

Leo-III is implemented in Scala and its source code, and that of related projects presented
in this article, is publicly available under BSD-3 license on GitHub.\footnote{
    See the individual projects related to the Leo prover family
    at \url{https://github.com/leoprover}. Further information are available at 
    \url{http://inf.fu-berlin.de/~lex/leo3}.
}
Installing Leo-III does not require any special libraries, apart from a reasonably
current version of the JDK and Scala. Also, Leo-III is readily available via
the SystemOnTPTP web interface~\cite{DBLP:journals/jar/Sutcliffe17}, and
it can be called via Sledgehammer~\cite{DBLP:journals/jar/BlanchetteBP13},
from the interactive proof assistant Isabelle/HOL~\cite{DBLP:books/sp/NipkowPW02}
for automatically discharging user's proof goals.

In a recent evaluation study of 19 different first-order and higher-order ATP
systems, Leo-III was found the most versatile (in terms of supported
logic formalisms) and best performing ATP system overall~\cite{grunge}.

This article presents a consolidated summary of previous conference and workshop
contributions~\cite{DBLP:conf/lpar/BenzmullerSW17,%
DBLP:conf/cade/SteenB18,DBLP:conf/icms/SteenWB16,%
DBLP:conf/lpar/SteenWB17,DBLP:conf/mkm/WisniewskiSB15}
as well as contributions from the first author's PhD thesis~\cite{steen2018}.
It is structured as follows: \S\ref{sec:holatp} briefly introduces HOL and 
summarizes challenging automation aspects. In \S\ref{sec:paramod} the basic paramodulation
calculus that is extended by Leo-III is presented, and practically
motivated extensions that are implemented in the prover are outlined. Subsequently,
the implementation of Leo-III is described in more detail in \S\ref{sec:impl},
and \S\ref{sec:ncl} presents the technology that enables Leo-III to reason
in various non-classical logics. An evaluation of Leo-III on a heterogeneous
set of benchmarks, including problems in non-classical logics,
is presented in \S\ref{sec:eval}.
Finally, \S\ref{sec:conclusion} concludes this article and sketches further work.


\section{Higher-Order Theorem Proving}\label{sec:holatp}
The term higher-order logic refers to expressive logical formalisms that allow
for quantification over predicate and function variables; such a logic
was first studied by Frege in the 1870s~\cite{GlossarWiki:Frege:1879}.
An alternative and more handy formulation was proposed by Church
in the 1940s~\cite{DBLP:journals/jsyml/Church40}.
He defined a higher-order logic on 
top of the simply typed $\lambda$-calculus.
His particular formalism,
referred to as Simple Type Theory (STT), was later further studied and refined
by Henkin~\cite{DBLP:journals/jsyml/Henkin50},
Andrews~\cite{DBLP:journals/jsyml/Andrews71,DBLP:journals/jsyml/Andrews72,DBLP:journals/jsyml/Andrews72a}
and others~\cite{DBLP:journals/jsyml/BenzmullerBK04,DBLP:journals/jsyml/Muskens07}.
In the remainder, the term HOL is used synonymously to Henkin's Extensional Type Theory
(ExTT)~\cite{DBLP:series/hhl/BenzmullerM14}; it constitutes the foundation of many contemporary
higher-order automated reasoning systems.
HOL provides lambda-notation as an
elegant and useful means to denote unnamed functions, predicates and
sets (by their characteristic functions), and it comes with built-in principles
of Boolean and functional extensionality as well as type-restricted comprehension.

A more in-depth presentation of HOL, its historical development, metatheory and
automation is provided by Benzmüller and Miller~\cite{DBLP:series/hhl/BenzmullerM14}.

\paragraph{Syntax and Semantics.}
HOL is a typed logic; every term of HOL is associated a fixed and unique type, written
as subscript. 
The set $\types$ of simple types is freely generated from a non-empty set $S$
of sort symbols (base types) and juxtaposition $\nu\tau$ of two types $\tau,\nu \in \types$,
the latter
denoting the type of functions from objects of type $\tau$ to objects of type $\nu$.
Function types are assumed to associate to the left and parentheses may be dropped if
consistent with the intended reading.
The base types are usually chosen to be $S := \{\iota, o\}$,
where $\iota$ and $o$ represent the type of individuals and the type of Boolean truth values,
respectively.

Let $\Sigma$ be a typed signature and let $\Var$ denote a set of typed variable symbols
such that there exist infinitely many variables for each type. Following
Andrews~\cite{andrews2002introduction}, it is assumed that the only primitive logical
connective is equality, denoted $=^\tau_{o\tau\tau} \in \Sigma$, for each type $\tau \in \types$ (called $Q$
by Andrews).
In the extensional setting of HOL, all remaining logical connectives such as disjunction $\lor_{ooo}$,
conjunction $\land_{ooo}$, negation $\neg_{oo}$, etc., can be defined in terms of them.
The terms of HOL are given by the following abstract syntax (where $\tau,\nu \in \types$ are types):
\begin{equation*}
s,t ::= c_\tau \in \Sigma \; | \; X_\tau \in \Var \; | \; \left(\lambda X_\tau.\, s_\nu\right)_{\nu\tau}
  \; | \; \left(s_{\nu\tau} \; t_\tau\right)_\nu
\end{equation*}
The terms are called {constants}, {variables}, {abstractions}
and {applications}, respectively. 
Application is assumed to associate to the left and parentheses may again be dropped whenever possible.
Conventionally, vector notation
$f_{\nu\tau^n \cdots \tau^1} \; \overline{t^i_{\tau^i}}$ is used to abbreviate
nested applications $\big(f_{\nu\tau^n \cdots \tau^1} \; t^1_{\tau^1} \; \cdots \; t^n_{\tau^n}\big)$,
where $f$ is a function term and the $t^i$, $1 \leq i \leq n$, are argument terms
of appropriate types.
The type of a term may be dropped for legibility reasons if obvious from the context.
Notions such as $\alpha$-, $\beta$-, and $\eta$-conversion, denoted
$\longrightarrow_\star$, for $\star \in \{\alpha, \beta, \eta\}$,
free variables $\fv(.)$ of a term, etc., are
defined as usual~\cite{DBLP:books/daglib/0032840}. 
\hl{The notion $s\{t/X\}$ is used to denote the (capture-free) substitution 
of variable $X$ by term $t$ in $s$.}
Syntactical equality between HOL terms, denoted $\equiv_\star$,
for $\star \subseteq \{\alpha, \beta, \eta\}$, is defined with respect to the assumed
underlying conversion rules.
Terms $s_o$ of type $o$ are {formulas}, and they are {sentences} if they are
closed. By convention, infix notation for fully applied logical connectives 
is used, e.g. $s_o \lor t_o$ instead of $(\lor_{ooo} \; s_o) \; t_o$.
 
As a consequence of G\"odel's Incompleteness Theorem, HOL with standard semantics
is necessarily incomplete. In contrast, theorem proving in HOL is usually considered with respect
to so-called general semantics (or Henkin semantics) in which a meaningful notion of
completeness can be achieved~\cite{DBLP:journals/jsyml/Henkin50,DBLP:journals/jsyml/Andrews72a}. 
The usual notions of general model structures, validity in these structures and related notions
are assumed in the following. Note that we do not assume that the general model structures
validate choice.
Intensional models have been described by Muskens~\cite{DBLP:journals/jsyml/Muskens07}
and studies of further general notions of semantics have been presented by
Andrews~\cite{DBLP:journals/jsyml/Andrews71} and
Benzm\"uller et al.~\cite{DBLP:journals/jsyml/BenzmullerBK04}.

\paragraph{Challenges to HOL Automation.}

HOL validates functional and Boolean extensionality principles,
referred to as $\mathrm{EXT}^{\nu\tau}$ and $\mathrm{EXT}^{o}$.
These principles
can be formulated within HOL's term language as
\begin{equation*}\begin{split}
\mathrm{EXT}^{\nu\tau} &:= \forall F_{\nu\tau}.\,\forall G_{\nu\tau}.\,(\forall X_\tau.\,F \; X =^\nu G \; X) \Rightarrow F =^{\nu\tau} G\\
\mathrm{EXT}^{o} &:= \forall P_o.\,\forall Q_o.\,(P \Leftrightarrow Q) \Rightarrow P =^o Q
\end{split}\end{equation*}
These principles state that two functions are equal if they correspond on every argument,
and that two formulas are equal if they are equivalent
(where $\Leftrightarrow_{ooo}$ denotes equivalence), respectively.
Using these principles, one can infer that two functions such as $\lambda P_o.\, \syntrue$
and $\lambda P_o.\, P \lor \neg P$ are in fact equal (where $\syntrue$ denotes syntactical truth),
and that $\left(\lambda P_o.\,\lambda Q_o.\,P \lor Q\right) = \left(\lambda P_o.\,\lambda Q_o.\, Q \lor P\right)$
is a theorem. 
Boolean Extensionality, in particular, poses a considerable challenge for HOL automation:
Two terms may be equal,
and thus subject to generating inferences, if 
the equivalence of all Boolean-typed subterms can be inferred. As a consequence,
a complete implementation of
non-ground proof calculi that make use of higher-order unification procedures
cannot simply use syntactical unification for locally deciding which inferences
are to be generated. 
In contrast to first-order theorem proving, it is hence
necessary to interleave syntactical unification and (semantical) proof search, which is
more difficult to control in practice.

As a further complication, higher-order unification is only semi-decidable and not
unitary~\cite{huet1973undecidability,goldfarb1981undecidability}.
It is not clear how many and which unifiers produced by a higher-order unification routine should
be chosen during proof search, and the unification procedure may never terminate on non-unifiable
terms.

In the context of first-order logic with equality, 
superposition-based calculi have proven
an effective basis for reasoning systems and provide a powerful notion of redundancy~\cite{DBLP:conf/cade/BachmairG90,DBLP:conf/cade/NieuwenhuisR92,DBLP:journals/logcom/BachmairG94}.
Reasoning with equality can also be addressed, e.g., by an RUE resolution
approach~\cite{DBLP:journals/jacm/DigricoliH86} and, in the higher-order case,
by reducing equality to equivalent formulas not containing
the equality predicate~\cite{DBLP:conf/cade/Benzmuller99}, as done in LEO.
The latter approaches however lack effectivity in practical applications
of large scale equality reasoning.

There are further practical challenges as there are only few implementation techniques
available for efficient data structures and indexing methods.
This hampers the effectivity of HOL reasoning systems and their application in practice.

\paragraph{HOL ATP Systems.} Next to the LEO prover family~\cite{DBLP:conf/cade/BenzmullerK98a,DBLP:journals/jar/BenzmullerSPT15,DBLP:conf/cade/SteenB18}, there
are further HOL ATP systems available: This includes
TPS~\cite{DBLP:journals/japll/AndrewsB06} as one of the earliest systems,
as well as Satallax~\cite{DBLP:conf/cade/Brown12},
coqATP~\cite{DBLP:series/txtcs/BertotC04}, agsyHOL~\cite{DBLP:conf/cade/Lindblad14}
and the higher-order (counter)model finder Nitpick~\cite{DBLP:conf/itp/BlanchetteN10}.
Additionally, there is ongoing work on extending the first-order theorem prover Vampire
to full higher-order reasoning~\cite{DBLP:conf/cade/BhayatR18,bhayat2019},
and some interactive proof assistants
such as Isabelle/HOL~\cite{DBLP:books/sp/NipkowPW02} can also be used for
automated reasoning in HOL. 
Further related systems include higher-order extensions
of SMT solvers~\cite{barbosa2019extending}, and there is ongoing work
to lift first-order ATP systems based on superposition to fragments of HOL, 
including E~\cite{Schulz:2002:EBT:1218615.1218621,DBLP:conf/tacas/VukmirovicBCS19}
and Zipperposition~\cite{DBLP:phd/hal/Cruanes15,DBLP:conf/cade/BentkampBCW18}.

Further notable higher-order reasoning systems include proof assistants such as 
PVS~\cite{DBLP:conf/cade/OwreRS92}, Isabelle/HOL, the HOL prover family including HOL4~\cite{gordon1993introduction},
and the HOL Light system~\cite{DBLP:conf/tphol/Harrison09a}. In contrast to ATP systems, proof assistants
do not finds proofs automatically but are rather used to formalize and verify hand-written proofs for
correctness.

\paragraph{Applications.}
The expressivity of higher-order logic has been exploited
for encoding various expressive non-classical logics within HOL.
Semantical embeddings of, among others, higher-order modal
logics~\cite{DBLP:journals/lu/BenzmullerP13,DBLP:conf/lpar/GleissnerSB17},
conditional logics~\cite{DBLP:journals/jphil/Benzmuller17},
many-valued logics~\cite{steen2016sweet}, deontic logic~\cite{DBLP:conf/deon/BenzmullerFP18},
free logics~\cite{DBLP:conf/icms/BenzmullerS16}, and combinations
of such logics~\cite{DBLP:journals/amai/Benzmuller11} can be used to
automate reasoning within those logics
using ATP systems for classical HOL.
A prominent result from the applications of automated reasoning in non-classical
logics, here in quantified modal logics, was the detection of a major flaw
in G\"odel's Ontological Argument~\cite{DBLP:journals/afp/FuenmayorB17,J36}
as well as the verification of Scott's variant of that argument~\cite{DBLP:conf/ijcai/BenzmullerP16}
using LEO-II and Isabelle/HOL.
Similar and further enhanced techniques were used to assess
foundational questions in metaphysics~\cite{DBLP:journals/lu/BenzmullerWP17,J47}. 

Additionally, Isabelle/HOL and the Nitpick system were used
to assess the correctness of concurrent C++ programs against
a previously formalized memory model~\cite{DBLP:conf/ppdp/BlanchetteWBOS11}.
The higher-order proof assistant HOL Light
played a key role in the verification of
Kepler's conjecture within the
Flyspeck project~\cite{DBLP:journals/corr/HalesABDHHKMMNNNOPRSTTTUVZ15}.
  

\section{Extensional Higher-Order Paramodulation} \label{sec:paramod}
Leo-III is a refutational ATP system. The initial, possibly empty, set of axioms and
the negated conjecture are transformed into an equisatisfiable set of formulas
in clause normal form (CNF), which is then iteratively saturated
until the empty clause is found.
Leo-III extends the complete, paramodulation-based calculus \calculus\ for HOL
(cf. \S\ref{ssec:ep}) with practically motivated, partly heuristic inference
rules.
Paramodulation extends resolution by a native treatment of equality at the calculus level.
In the context of first-order logic, it was developed in the late 1960s by G. Robinson and
L. Wos~\cite{robinson1969paramodulation}
as an attempt to overcome the shortcomings of resolution-based approaches 
to handling equality. A paramodulation inference
incorporates the principle of {replacing equals by equals} and can
be regarded as a speculative conditional rewriting step.
In the context of first-order theorem proving, superposition-based calculi
\cite{DBLP:conf/cade/BachmairG90,DBLP:conf/cade/NieuwenhuisR92,DBLP:journals/logcom/BachmairG94}
improve the naive paramodulation approach by imposing ordering restrictions on the
inference rules such that only a relevant subset of all possible inferences are generated.
However, due to the more complex structure of the term language of HOL,
there do not exist suitable term orderings that allow a straightforward adaption
of this approach to the higher-order setting.\footnote{
  As a simple counterexample, consider a (strict) term ordering $\succ$ for HOL terms
  that satisfies the usual properties from first-order superposition (e.g.,
  the subterm property) and is compatible with $\beta$-reduction. For
  any nonempty signature $\Sigma$, $\mathrm{c} \in \Sigma$, the chain 
  $\mathrm{c} \equiv_\beta (\lambda X.\, \mathrm{c}) \; \mathrm{c} \succ \mathrm{c}$
  can be constructed, implying $\mathrm{c} \succ \mathrm{c}$ and thus contradicting irreflexivity of $\succ$.
  Note that $(\lambda X.\, \mathrm{c}) \; \mathrm{c} \succ \mathrm{c}$
  since the right-hand side is a proper subterm of the left-hand side
  (assuming an adequately lifted definition of subterm property to HO terms).
} However, there is recent work to overcome this situation by relaxing 
restrictions on the employed orderings~\cite{DBLP:conf/cade/BentkampBTVW19}.

\subsection{The EP Calculus\label{ssec:ep}}
Higher-order paramodulation for extensional type theory was first presented by
Benzm\"uller~\cite{DBLP:phd/dnb/Benzmuller99,DBLP:conf/cade/Benzmuller99}.
This calculus was mainly theoretically motivated and extended a resolution calculus with a
paramodulation rule instead of being based on a paramodulation rule alone. Additionally,
that calculus contained a rule that expanded equality literals
by their definition due to Leibniz.\footnote{
  The {Identity of Indiscernibles} (also known as {Leibniz's law}) refers to a principle
  first formulated
  by Gottfried Leibniz in the context of theoretical philosophy~\cite{Leibniz1989}.
  The principle states that if two objects $X$ and $Y$ coincide on every property $P$,
  then they are equal,
  i.e. $\forall X_\tau.\,\forall Y_\tau.\,\left(\forall P_{o\tau}. P \; X \Leftrightarrow P \; Y\right) \Rightarrow X = Y$,
  where ''$=$''
  denotes the desired equality predicate.
  Since this principle can easily be formulated in HOL, it is possible
    to encode equality in higher-order logic without using the primitive equality
    predicate. An extensive analysis of the intricate differences between
    primitive equality and defined notions of equality is presented by
    Benzm\"uller et al.~\cite{DBLP:journals/jsyml/BenzmullerBK04} to which the authors refer for further details.
}
As Leibniz equality formulas effectively enable
cut-simulation~\cite{DBLP:journals/corr/abs-0902-0043}, the proposed calculus seems unsuited for automation.
The calculus \calculus\ presented in the following, in contrast, avoids the expansion of equality predicates but
adapts the use of dedicated calculus rules for extensionality principles from Benzmüller~\cite{DBLP:conf/cade/Benzmuller99}.

\hl{An equation, denoted $s \simeq t$, is a pair of HOL terms of the same type,
where $\simeq$ is assumed to be symmetric (i.e., $s \simeq t$ represents
both $s \simeq t$ and $t \simeq s$).}
A literal $\ell$ is a signed equation, written $[s \simeq t]^\alpha$,
where $\alpha \in \{\pospol,\negpol\}$ is the polarity of $\ell$.
Literals of form $[s_o]^\alpha$ are shorthand for $[s_o \simeq \syntrue]^\alpha$, and
negative literals $[s \simeq t]^\negpol$ are also referred to as unification constraints.
A negative literal $\ell$ is called a {flex-flex} unification constraint
if $\ell$ is of the
form $[X \; \overline{s^i} \simeq Y \; \overline{t^j}]^\negpol$,
where $X,Y$ are variables.
A clause $\clause{C}$ is a multiset of literals, denoting its disjunction.
For brevity, if $\clause{C},\clause{D}$ are 
clauses and $\ell$ is a literal, $\clause{C} \lor \ell$
and $\clause{C} \lor \clause{D}$ denote the multi-union 
$\clause{C} \cup \{ \ell \}$ and $\clause{C} \cup \clause{D}$, respectively.
$s|_\pi$ is the subterm of $s$ at position $\pi$, and $s[r]_\pi$ denotes the term
that is obtained by replacing the subterm of $s$ at position $\pi$
by $r$; \hl{$\alpha$-conversion is applied implicitly whenever necessary to avoid variable
capture.}

The \calculus\ calculus can be divided into four groups of inference rules:
\begin{description}
  \item[\emph{Clause normalization.}]
    The clausification rules of \calculus\ are mostly standard, cf. Fig.~\ref{fig:calculus:add}. Every
    non-normal clause is transformed into an equisatisfiable set of clauses in CNF.
    Multiple conclusions are written one below the other. 
    Note that the clausification rules are proper inference rules rather than a dedicated
    meta operation. This is due to the fact that non-CNF clauses may be generated from
    the application of the remaining inferences rules, hence renormalization during
    proof search may be necessary. 
    In the following we use \cnfapp\ to refer to the entirety of the CNF rules.
    
    For the elimination of existential quantifiers, see rule (CNFExists) in
    Fig.~\ref{fig:calculus:add},
    the sound Skolemization technique
    of Miller~\cite{miller1983proofs,DBLP:journals/logcom/Miller91} is assumed.
  \item[\emph{Primary inferences.}]
    The primary inference rules of EP are {paramodulation} (Para),
    {equality factoring} (EqFac) and {primitive substitution} (PS),
    cf. Fig.~\ref{fig:calculus:add}.

    \begin{figure}[tb]
      \centering
      \fbox{
        \begin{minipage}{.95\textwidth}
          \fbox{\textsc{Clausification rules}} \\[1em]
          \begin{minipage}{.45\textwidth}
          \begin{prooftree}
            \AxiomC{$\clause{C} \lor [(l_\tau = r_\tau) \simeq \syntrue]^\alpha$}
            \RightLabel{\scriptsize(LiftEq)}
            \UnaryInfC{$\clause{C} \lor [l_\tau \simeq r_\tau]^\alpha$}
          \end{prooftree}
          \end{minipage}
          \begin{minipage}{.45\textwidth}
          \begin{prooftree}
            \AxiomC{$\clause{C} \lor [\neg s_o]^\alpha$}
            \RightLabel{\scriptsize(CNFNeg)}
            \UnaryInfC{$\clause{C} \lor [s_o]^{\overline{\alpha}}$}
          \end{prooftree}
          \end{minipage} \\[.5em]
          \begin{minipage}{.45\textwidth}
          \begin{prooftree}
            \AxiomC{$\clause{C} \lor [s_o \lor t_o]^\pospol$}
            \RightLabel{\scriptsize(CNFDisj)}
            \UnaryInfC{$\clause{C} \lor [s_o]^\pospol \lor [t_o]^\pospol$}
          \end{prooftree}
          \end{minipage}
          \begin{minipage}{.45\textwidth}
          \begin{prooftree}
            \AxiomC{$\clause{C} \lor [s_o \lor t_o]^\negpol$}
            \RightLabel{\scriptsize(CNFConj)}
            \UnaryInfC{$\clause{C} \lor [s_o]^\negpol$}
            \noLine
            \def\extraVskip{0em}
            \UnaryInfC{$\clause{C} \lor [t_o]^\negpol$}
          \end{prooftree}
          \end{minipage}\\[.1em]
          \begin{minipage}{.45\textwidth}
          \begin{prooftree}
            \AxiomC{$\clause{C} \lor [\forall X_\tau.\,s_o]^\pospol$}
            \RightLabel{\scriptsize(CNFAll)$^\dagger$}
            \UnaryInfC{\hl{$\clause{C} \lor [s_o\{Z/X\}]^\pospol$}}
          \end{prooftree}
          \end{minipage}
          \begin{minipage}{.45\textwidth}
          \begin{prooftree}
            \AxiomC{$\clause{C} \lor [\forall X_\tau.\,s_o]^\negpol$}
            \RightLabel{\scriptsize(CNFExists)$^\ddagger$}
            \UnaryInfC{\hl{$\clause{C} \lor [s_o\{\mathrm{sk} \; \overline{\fv(\clause{C})}/X\}]^\negpol$}}
          \end{prooftree}
          \end{minipage}
          \\[1em]
          \footnotesize
          \mbox{} 
          \hfill
          $\dagger$: where $Z_\tau$ is fresh for $\clause{C}$
          \qquad
          $\ddagger$: where $\mathrm{sk}$ is a new constant of appropriate type
          \\[-.5em]
          \hrule
          \vspace{.5em}
          \fbox{\textsc{Primary inference rules}}
          \begin{prooftree}
            \AxiomC{$\clause{C} \lor [s_\tau \simeq t_\tau]^\alpha$}
            \AxiomC{$\clause{D} \lor [l_\nu \simeq r_\nu]^\pospol$}
            \RightLabel{\scriptsize\paramodrule$^{\rotatebox[origin=c]{180}{$\dagger$}}$}
            \BinaryInfC{$[s[r]_\pi \simeq t]^\alpha \lor \clause{C} \lor \clause{D} \lor [s|_\pi \simeq l]^\negpol$}
          \end{prooftree}
          \begin{prooftree}
            \AxiomC{$\clause{C} \lor [s_\tau \simeq t_\tau]^\alpha \lor [u_\tau \simeq v_\tau]^\alpha$}
            \RightLabel{\scriptsize\factorrule}
            \UnaryInfC{$\clause{C} \lor [s_\tau \simeq t_\tau]^\alpha \lor [s_\tau \simeq u_\tau]^\negpol \lor [t_\tau \simeq v_\tau]^\negpol$}
          \end{prooftree}
          \begin{prooftree}
            \AxiomC{$\clause{C} \lor [H_{\tau} \; \overline{s^i_{\tau^i}}]^\alpha$}
            \AxiomC{$G \in \mathcal{GB}^{\{\neg, \lor\} \cup \{\Pi^\nu,\, {=}^\nu \, \mid \, \nu \in \types\}}_\tau$}
            \RightLabel{\scriptsize\primsubstrule}
            \BinaryInfC{$\clause{C} \lor [H_{\tau} \; \overline{s^i_{\tau^i}}]^\alpha \lor [H \simeq G]^\negpol$}
          \end{prooftree}
          \footnotesize
          \mbox{} 
          \hfill
          \rotatebox[origin=c]{180}{$\dagger$}: if $s|_\pi$ is of type $\nu$ and $\fv(s|_\pi) \subseteq \fv(s)$
        \end{minipage}
      }
      \caption{Extensionality and unification rules of \calculus.\label{fig:calculus:add}}    
    \end{figure}
    
    The paramodulation rule (Para) replaces subterms of literals within clauses
    by (potentially) equal terms given from a positive equality literal. Since the latter
    clause might not be a unit clause, the rewriting step can be considered {conditional}
    where the remaining literals represent the respective additional conditions.
    Factorization \factorrule\ contracts two literals that are semantically overlapping (i.e., 
    one is more general than the other) but not
    syntactically equal. This reduces the size of the clause, given that the unification
    of the respective two literals is successful.
    These two rules introduce {unification constraints} that are encoded as negative
    equality literals:
    A generating inference is semantically justified
    if the unification constraint(s) can be solved. Since higher-order unification is not decidable,
    these constraints are explicitly encoded into the result clause for subsequent analysis.
    Note that a resolution inference between clauses $\clause{C} \equiv \clause{C}^\prime \lor [p]^\pospol$
    and $\clause{D} \equiv \clause{D}^\prime \lor [p]^\negpol$ 
    can be simulated by the (Para) rule as the literals $[p]^\alpha$ are actually shorthands
    for $[p \simeq \syntrue]^\alpha$ and the clause $[\syntrue \simeq \syntrue]^\negpol
    \lor \clause{C}^\prime \lor \clause{D}^\prime \lor [p \simeq p]^\negpol$, which eventually simplifies
    to $\clause{C}^\prime \lor \clause{D}^\prime$, is generated.

    Moreover, note that both (Para) and (EqFac) are unordered and produce numerous redundant
    clauses. In practice, Leo-III tries to remedy this situation by using
    heuristics to restrict the number of generated clauses,
    including a higher-order term ordering, cf.~\S\ref{sec:impl}. Such heuristics, e.g.,
    prevent redundant paramodulation inferences between positive propositional literals 
    in clauses $\clause{C} \lor [p]^\pospol$ and $\clause{D} \lor [q]^\pospol$.

    The primitive substitution inference \primsubstrule\ approximates the logical structure of literals
    with flexible heads. In contrast to early attempts that blindly guessed a concrete
    instance of head variables, \primsubstrule\ uses so-called general bindings,
    denoted $\mathcal{GB}_\tau^{t}$~\cite[\S2]{DBLP:journals/jar/BenzmullerSPT15},
    to step-wise approximate the instantiated term structure and hence limit the 
    explosive growth of primitive substitution. Nevertheless, \primsubstrule\ still
    enumerates the whole universe of terms but, in practical applications,
    often few applications of the rule are sufficient to find a refutation.
    An example where
    primitive substitution is necessary is the following: Consider a clause $\clause{C}$
    given by $\clause{C} \equiv [P_o]^\pospol$ where $P$ is a Boolean variable. This clause
    corresponds to the formula $\forall P_o.\, P$ which is clearly not a theorem.
    Neither \paramodrule\ nor \factorrule\ or any other calculus rules presented so far or further below
    (except for \primsubstrule) allow to construct a derivation to the empty clause. However,
    using \primsubstrule, there exists a derivation
    $[P]^\pospol \vdash_{\text{\primsubstrule},\text{\bindrule}} [\neg P^\prime]^\pospol \vdash_{\text{CNFNeg}} [P^\prime]^\negpol$. 
    Note that $\{\neg P^\prime/P\}$ is a substitution applying a
    general binding from $\mathcal{GB}_o^{\{\neg, \lor\}\cup \{\Pi^\tau, =^\tau \, \mid \, \tau \in \types\}}$ that approximates logical negation. Now, a simple refutation
    involving $[P]^\pospol$ and $[P^\prime]^\negpol$ can be found.

  \item[\emph{Extensionality rules.}]
    The rules (NBE) and (PBE) -- for negative resp. positive Boolean extensionality --
    as well as (NFE) and (PFE) -- for negative resp. positive functional extensionality --
    are the extensionality rules of \calculus, cf. Fig.~\ref{fig:calculus}.
    
    While the functional extensionality rules gradually ground the literals to base types
    and provide witnesses for the (in-)equality of function symbols to the search space,
    the Boolean extensionality rules enable the application of clausification
    rules to the Boolean-typed sides of the literal, thereby lifting them into
    semantical proof search.
    These rules eliminate the need for explicit
    extensionality axioms in the search space, which would
    enable cut-simulation~\cite{DBLP:journals/corr/abs-0902-0043} and hence drastically hamper proof search.
  \item[\emph{Unification.}]
    The unification rules of \calculus\ are a variant of
    Huet's unification rules and presented in Fig.~\ref{fig:calculus}.
    They can be eagerly applied to the unification constraints in clauses. In an extensional
    setting, syntactical search for unifiers and
    semantical proof search coincide, and unification transformations are regarded proper
    calculus rules.
    As a result, the unification rules might only partly solve
    (i.e., simplify) unification constraints and unification constraints themselves are
    eligible to subsequent inferences. The bundled unification rules are referred to
    as \uniapp.
\end{description}

\begin{figure}[tb]
      \centering
      \fbox{
        \begin{minipage}{.95\textwidth}
          \fbox{\textsc{Extensionality rules}} \\[1em]
          \begin{minipage}{.49\textwidth}
          \begin{prooftree}
            \AxiomC{$\clause{C} \lor [s_o \simeq t_o]^\pospol$}
            \RightLabel{\scriptsize\boolposrule}
            \UnaryInfC{$\clause{C} \lor [s_o]^\pospol \lor [t_o]^\negpol$}
            \noLine
            \def\extraVskip{0em}
            \UnaryInfC{$\clause{C} \lor [s_o]^\negpol \lor [t_o]^\pospol$}
          \end{prooftree}
          \end{minipage}
          \begin{minipage}{.49\textwidth}
          \begin{prooftree}
            \AxiomC{$\clause{C} \lor [s_o \simeq t_o]^\negpol$}
            \RightLabel{\scriptsize\boolnegrule}
            \UnaryInfC{$\clause{C} \lor [s_o]^\pospol \lor [t_o]^\pospol$}
            \noLine
            \def\extraVskip{0em}
            \UnaryInfC{$\clause{C} \lor [s_o]^\negpol \lor [t_o]^\negpol$}
          \end{prooftree}
          \end{minipage} \\[.5em]
          \begin{minipage}{.49\textwidth}
          \begin{prooftree}
            \AxiomC{$\clause{C} \lor [s_{\nu\tau} \simeq t_{\nu\tau}]^\pospol$}
            \RightLabel{\scriptsize\funcposrule$^\dagger$}
            \UnaryInfC{$\clause{C} \lor [s \; X_\tau \simeq t \; X_\tau]^\pospol$}
          \end{prooftree}
          \end{minipage}
          \begin{minipage}{.49\textwidth}
          \begin{prooftree}
            \AxiomC{$\clause{C} \lor [s_{\nu\tau} \simeq t_{\nu\tau}]^\negpol$}
            \RightLabel{\scriptsize\funcnegrule$^\ddagger$}
            \UnaryInfC{$\clause{C} \lor [s \; \mathrm{sk}_\tau \simeq t \; \mathrm{sk}_\tau]^\negpol$}
          \end{prooftree}
          \end{minipage}
          \\[1em]
          \footnotesize
          \mbox{} 
          \hfill
          $\dagger$: where $X_\tau$ is fresh for $\clause{C}$
          \qquad
          $\ddagger$: where $\mathrm{sk}_\tau$ is a new Skolem term for 
          $\clause{C} \lor [s_{\nu\tau} \simeq t_{\nu\tau}]^\negpol$
          \\[-.5em]
          \hrule
          \vspace{.5em}
          \fbox{\textsc{Unification rules}}  \\[1em]
          \begin{minipage}{.49\textwidth}
          \begin{prooftree}
            \AxiomC{$\clause{C} \lor [s_\tau \simeq s_\tau]^\negpol$}
            \RightLabel{\scriptsize\trivrule}
            \UnaryInfC{$\clause{C}$}
          \end{prooftree}
          \end{minipage}
          \begin{minipage}{.49\textwidth}
          \begin{prooftree}
            \AxiomC{$\clause{C} \lor [X_\tau \simeq s_\tau]^\negpol$}
            \RightLabel{\scriptsize\bindrule$^\dagger$}
            \UnaryInfC{$\clause{C}\{s / X\}$}
          \end{prooftree}
          \end{minipage}\\[.5em]
          \begin{minipage}{\textwidth}
          \begin{prooftree}
               \AxiomC{$\clause{C} \lor 
                  [c \; \overline{s^i} \simeq c \; \overline{t^i}]^\negpol$}
               \RightLabel{\scriptsize\decomprule}
               \UnaryInfC{$\clause{C} \lor [s^1 \simeq t^1]^\negpol \lor \ldots \lor [s^n \simeq t^n]^\negpol$}
          \end{prooftree}
          \end{minipage}\\[.5em]
          \begin{minipage}{\textwidth}
          \begin{prooftree}
            \AxiomC{$\clause{C} \lor [X_{\nu\overline{\mu}} \; \overline{\,s^i\,} \simeq c_{\nu\overline{\tau}} \; \overline{\,t^j\,}]^\negpol$}
            \AxiomC{$g_{\nu\overline{\mu}} \in \mathcal{GB}^{\{c\}}_{\nu\overline{\mu}}$}
            \RightLabel{\scriptsize\flexrigidrule}
            \BinaryInfC{$\clause{C} \lor [X_{\nu\overline{\mu}} \; \overline{\,s^i\,} \simeq c_{\nu\overline{\tau}} \; \overline{\,t^j\,}]^\negpol \lor [X \simeq g]^\negpol$}
          \end{prooftree}
          \end{minipage}\\[.5em]
          \begin{minipage}{\textwidth}
          \begin{prooftree}
            \AxiomC{$\clause{C} \lor [X_{\nu\overline{\mu}} \; \overline{\,s^i\,} \simeq Y_{\nu\overline{\tau}} \; \overline{\,t^j\,}]^\negpol$}
            \AxiomC{$g_{\nu\overline{\mu}} \in \mathcal{GB}^{\{h\}}_{\nu\overline{\mu}}$}
            \RightLabel{\scriptsize\flexflexrule$^\ddagger$}
            \BinaryInfC{$\clause{C} \lor [X_{\nu\overline{\mu}} \; \overline{\,s^i\,} \simeq Y_{\nu\overline{\tau}} \; \overline{\,t^j\,}]^\negpol \lor [X \simeq g]^\negpol$}
          \end{prooftree}
          \end{minipage}
          \\[1em]
          \footnotesize
          \mbox{} 
          \hfill
          $\dagger$: where $X_\tau \notin \fv(s)$
          \qquad
          \hl{$\ddagger$: where $h \in \Sigma$ is a constant}
        \end{minipage}
      }
      \caption{Primary inference rules and extensionality rules 
      of \calculus. \label{fig:calculus}}
    \end{figure}

\noindent A set $\Phi$ of sentences has a {refutation in \calculus}, denoted
$\Phi \vdash \square$, iff the empty clause can be derived in \calculus.
A clause is the {empty clause}, written $\emptyclause$,
if it contains no or only
flex-flex unification constraints.
This is motivated
by the fact that flex-flex unification problems can always be solved,
and hence any clause consisting of only
flex-flex constraints is necessarily unsatisfiable~\cite{DBLP:journals/jsyml/Henkin50}.

\begin{theorem}[Soundness and Completeness of $\calculus$]
EP is sound and refutationally complete for HOL with Henkin semantics.
\end{theorem}
\begin{proof}
Soundness is guaranteed by~\cite[Theorem 3.8]{steen2018}.
For completeness, the following argument is used 
(cf.~\cite[\S 3]{steen2018} for the detailed
definitions and notions):
By~\cite[Lemma 3.24]{steen2018} the set of sentences that cannot be refuted in \calculus\
is an abstract consistency class~\cite[Def. 3.11]{steen2018}, and by~\cite[Lemma 3.17]{steen2018}
this abstract consistency class can be
extended to a Hintikka set.
The existence of a Henkin model that satisfies this Hintikka set is guaranteed 
by~\cite[Theorem 5]{reductionHintikka2020}.
\qed
\end{proof}

An example for a refutation in EP is given in the following:

\begin{example}[Cantor's Theorem]
Cantor's Theorem states that, given a set $A$, the power set of $A$
has a strictly greater cardinality than $A$ itself.
The core argument of the proof can be formalized as follows:
\begin{equation*}
  \neg \exists f_{o\iota\iota}.\, \forall Y_{o\iota}.\,\exists X_\iota.\, f\;X = Y
  \tag{$\mathbf{C}$}
  \label{eq:cantor:surjective}
\end{equation*}
Formula \ref{eq:cantor:surjective} states that there exists no surjective function $f$ from a set
to its power set.        
A proof of \ref{eq:cantor:surjective} in EP makes use of
functional extensionality, Boolean extensionality, primitive substitution
as well as nontrivial higher-order pre-unification; it is given below.

By convention,
the application of a calculus rule (or of a compound rule) is stated with the respective
premise clauses enclosed in parentheses after the rule name. For rule \primsubstrule,
the second argument describes which general binding was used for the instantiation; e.g.,
$\primsubstapp(\clause{C}, \approxbinding{\tau}{t})$ denotes an instantiation
with an approximation of term $t$ for goal type $\tau$.\footnote{\hl{The set $\approxbinding{\tau}{t}$ of
  {approximating/partial bindings} parametric to a type
  $\tau = {\beta{\alpha^n\ldots\alpha^1}}$ (for $n\geq 0$) and to a
  constant $t$ of type ${\beta{\gamma^m\ldots\gamma^1}}$ (for $m\geq 0$)
  is defined as follows (for further details see also~\cite{SnGa89}): Given
  a ``name'' $k$ (where a {name} is either a constant or a variable)
  of type ${\beta{\gamma^m\ldots\gamma^1}}$, the term ${l}$ having form
  $\lambda {X^1_{\alpha^{1}}. \ldots \lambda X^n_{\alpha^{n}}}. (k\,
  {{r^1}\ldots {r^m}})$ is a
  {partial binding} of type ${\beta{\alpha^n\ldots \alpha^1}}$ and {head}
  $k$. Each ${r}^{i \leq m}$ has form $H^i {X^1_{\alpha^1}\ldots X^n_{\alpha^n}}$
  where $H^{i\leq m}$ are fresh variables typed
  $\gamma^{i\leq m}{\alpha^n\ldots \alpha^1}$.  {Projection bindings}
  are partial bindings whose head $k$ is one of $X^{i\leq l}$.
  {Imitation bindings} are partial bindings whose head $k$ is
  identical to the given constant symbol $t$ in the superscript of
  $\approxbinding{\tau}{t}$.
  $\approxbinding{\tau}{t}$ is the set of all
  projection and imitation bindings modulo type
  $\tau$ and constant $t$. In our example derivation we twice use 
  imitation bindings of form $\lambda X_\iota. \neg(H_{o\iota} X_\iota)$ from $\approxbinding{o\iota}{\neg}$.}}
        {\arraycolsep=2.5pt
        \begin{equation*}
        \begin{array}{ll}
          \cnfapp(\neg\text{\ref{eq:cantor:surjective}}) \colon &
            \clause{C}_1 \colon [\sk^1 \; (\sk^2 \; \FV^1)\simeq \FV^1]^\pospol \\
          \funcposapp(\clause{C}_1) \colon &
            \clause{C}_2 \colon [\sk^1 \; (\sk^2 \; \FV^1) \; \FV^2 \simeq \FV^1 \; \FV^2]^\pospol \\
          \boolposapp(\clause{C}_2) \colon &
            \clause{C}_3 \colon [\sk^1 \; (\sk^2 \; \FV^1) \; \FV^2]^\pospol \lor
                                [\FV^1 \; \FV^2]^\negpol \\
            & \clause{C}_4 \colon [\sk^1 \; (\sk^2 \; \FV^3) \; \FV^4]^\negpol \lor
                                  [\FV^3 \; \FV^4]^\pospol  \\
          \primsubstapp(\clause{C}_3, \approxbinding{o\iota}{\neg}), \cnfapp \colon &
             \clause{C}_5 \colon [\sk^1 \; \big(\sk^2 \; (\lambda Z_\iota.\, \neg (\FV^5 \; Z))\big) \; \FV^2]^\pospol \lor
                                [\FV^5 \; \FV^2]^\pospol\\
          \primsubstapp(\clause{C}_4, \approxbinding{o\iota}{\neg}), \cnfapp \colon &
             \clause{C}_6 \colon [\sk^1 \; \big(\sk^2 \; (\lambda Z_\iota.\, \neg (\FV^6 \; Z))\big) \; \FV^4]^\negpol \lor
                                  [\FV^6 \; \FV^4]^\negpol  \\
          \factorapp(\clause{C}_5), \uniapp \colon &
            \clause{C}_7 \colon [\sk^1 \;
                                  \big(\sk^2 \; \lambda Z_\iota.\, \neg (\sk^1 \; Z \; Z)\big) \;
                                  \big(\sk^2 \; \lambda Z_\iota.\, \neg (\sk^1 \; Z \; Z)\big)]^\pospol \\
          \factorapp(\clause{C}_6), \uniapp \colon &
            \clause{C}_8 \colon [\sk^1 \;
                                  \big(\sk^2 \; \lambda Z_\iota.\, \neg (\sk^1 \; Z \; Z)\big) \;
                                  \big(\sk^2 \; \lambda Z_\iota.\, \neg (\sk^1 \; Z \; Z)\big)]^\negpol \\
          \paramodapp(\clause{C}_7, \clause{C}_8), \uniapp \colon &
            \emptyclause
        \end{array}
        \end{equation*}
        }\\
The Skolem symbols $\sk^1$ and $\sk^2$ used in the above proof have type
$o\iota\iota$ and $\iota(o\iota)$, respectively and the $\FV^i$ denote fresh free variables
of appropriate type.
A unifier $\sigma_{\clause{C}_7}$ generated by \uniapp\ for producing $\clause{C}_7$ is
given by (analogously for $\clause{C}_8$):
\begin{equation*}
  \sigma_{\clause{C}_7} \equiv \Big\{
    \sk^2 \; \big(\lambda Z_\iota.\, \neg (\sk^1 \; Z \; Z)\big)/\FV^2,
    \big(\lambda Z_\iota.\, \sk^1 \; Z \; Z\big)/\FV^5
  \Big\}
\end{equation*}
Note that, together with the substitution
$\sigma_{\clause{C}_3} \equiv \big\{\lambda Z_\iota.\, \neg (\FV^5 \; Z) / \FV^1\big\}$
generated by approximating
$\neg_{oo}$ via \primsubstrule\ on $\clause{C}_3$, the free variable $\FV^1$
in $\clause{C}_1$ is instantiated
by $\sigma_{\clause{C}_7} \circ \sigma_{\clause{C}_3}(\FV^1) \equiv
\lambda Z_\iota.\, \neg (\sk^1 \; Z \; Z)$.
Intuitively, this instantiation encodes the {diagonal set of $\sk^1$},
given by $\{x \mid x \notin \sk^1(x) \}$, as used in the traditional proofs of Cantor's Theorem;
see, e.g., Andrews~\cite{andrews1984automating}.

The TSTP representation of Leo-III's proof for this problem is presented in Fig.~\ref{fig:output}.
\end{example}

\subsection{Extended Calculus}\label{ssec:extcalculus}
As indicated further above, Leo-III targets automation of HOL with Henkin semantics and choice.
The calculus EP from \S\ref{ssec:ep}, however, does not address choice. To this end,
Leo-III implements several additional calculus rules that either accommodate further reasoning
capabilities (e.g.\, reasoning with choice) or address the inevitable explosive growth of the search
space during proof search (e.g.\, simplification routines).
The latter kind of rules are practically motivated, partly heuristic, and hence primarily
target technical issues that complicate effective automation in practice. Note that no completeness claims
are made for the extended calculus with respect to HOL with choice. Furthermore, some of the below
features of Leo-III are unsound with respect to HOL without choice. Since, however, Leo-III is designed
as a prover for HOL with choice, this is not a concern here.

The additional rules address the follows aspects \hl{(see an earlier
paper~\cite{DBLP:conf/cade/SteenB18} and Steen's doctoral thesis~\cite{steen2018}
for details and examples)}:

\begin{description}
  \item[\emph{Improved clausification.}]
    Leo-III employs {definitional clausification}~\cite{DBLP:conf/cade/WisniewskiSKB16}
    to reduce the number of clauses created during clause normalization.
    Moreover, {miniscoping} is employed prior to clausification.
  \item[\emph{Clause contraction.}]
  Leo-III implements equational {simplification procedures}, including {subsumption},
  destructive {equality resolution}, heuristic {rewriting} and contextual {unit cutting}
  (simplify-reflect)~\cite{Schulz:2002:EBT:1218615.1218621}.
  \item[\emph{Defined equalities.}]
  Common notions of defined equality predicates (Leibniz equality, Andrews equality) are heuristically
  replaced with {primitive equality} predicates.
  \item[\emph{Choice.}]
  Leo-III implements additional calculus rules for reasoning with choice.
  \item[\emph{Function synthesis.}]
  If plain unification fails for a set of unification constraints,
  Leo-III may try to synthesize functions that meet the specifications represented
  by the unification constraint. This is done using
  special choice instances that simulate if-then-else terms which explicitly enumerate
  the desired input output relation of that function.
  In general, this rule tremendously increases the search
  space, but it also enables Leo-III to solve some hard problems with TPTP rating 1.0
  that were not solved by any ATP system before.
  \item[\emph{Injective functions.}]
  Leo-III addresses improved reasoning with injective functions by postulating
  the existence of left inverses for function symbols that are inferred to be
  injective, see also below.
  \item[\emph{Further rules.}]
  Prior to clause normalization, Leo-III might {instantiate universally quantified}
  variables with heuristically chosen terms. 
  This includes the exhaustive instantiation of finite types (such as $o$, $oo$, etc.)
  as well as partial instantiation for other interesting types (such as $o\tau$ for some
  type $\tau$).
  \end{description}
  The addition of the above calculus rules to EP in Leo-III enables the system to solve
  various problems that can otherwise not be solved (in reasonable resource limits).
  An example problem that could not be solved by any higher-order ATP system before
  is the following, cf.~\cite[Problem \textbf{X5309}]{DBLP:conf/cade/AndrewsBB00}:

\begin{example}[Cantor's Theorem, revisited]
  Another possibility to encode Cantor's theorem is by using a formulation based on injectivity:
  \begin{equation}
    \neg \big(\exists f_{\iota(o\iota)}.\, \forall X_{o\iota}.\,\forall Y_{o\iota}.\, (f\;X = f\;Y) \Rightarrow X = Y\big)
    \tag{$\mathbf{C^\prime}$}
    \label{eq:cantor:injective}
  \end{equation}
  Here, the nonexistence of an injective function from a set's power set to the original
  set is postulated. This conjecture can easily be proved using Leo-III's injectivity rule $(\mathrm{INJ})$
  that, given a fact stating that some function symbol $f$ is injective, introduces the left inverse
  of $f$, say $f^{\mathrm{inv}}$, as fresh function symbol to the search space.
  The advantage is that $f^{\mathrm{inv}}$ is then available to subsequent inferences and can 
  act as an explicit evidence for the (assumed) existence of such a function which is then refuted.
  The full proof of \ref{eq:cantor:injective} is as follows:\\

  \noindent\resizebox{\textwidth}{!}{$\arraycolsep=2.5pt
        \begin{array}{ll}
          \cnfapp(\neg\text{\ref{eq:cantor:injective}}) \colon &
            \clause{C}_0 \colon [\sk \; \FV^1 \simeq \sk \; \FV^2]^\negpol \lor [\FV^1 \simeq \FV^2]^\pospol \\
          \funcposapp(\clause{C}_0) \colon &
            \clause{C}_1 \colon [\sk \; \FV^1 \simeq \sk \; \FV^2]^\negpol \lor [\FV^1 \; \FV^3 \simeq \FV^2 \; \FV^3]^\pospol \\
          \mathrm{INJ}(\clause{C}_0) \colon &
            \clause{C}_2 \colon [\sk^{\mathrm{inv}} \; (\sk \; \FV^4) \simeq \FV^4]^\pospol  \\
          \funcposapp(\clause{C}_2) \colon &
            \clause{C}_3 \colon [\sk^{\mathrm{inv}} \; (\sk \; \FV^4) \; \FV^5 \simeq \FV^4 \; \FV^5]^\pospol \\
          \paramodapp(\clause{C}_3,\clause{C}_1) \colon &
            \clause{C}_4 \colon [\sk \; \FV^1 \simeq \sk \; \FV^2]^\negpol \lor
                                [\FV^1 \; \FV^3 \simeq \FV^4 \; \FV^5]^\pospol \lor\\
                                &
                                \hspace{1.7em} [\sk^{\mathrm{inv}} \; (\sk \; \FV^4) \; \FV^5 \simeq \FV^2 \; \FV^3]^\negpol \\
          \uniapp(\clause{C}_4) \colon &
            \clause{C}_5 \colon [\sk^{\mathrm{inv}} \; \big(\sk \; (\FV^7 \; \FV^3)\big) \; (\FV^6 \; \FV^3)
                                    \simeq \FV^7 \; \FV^3 \; (\FV^6 \;  \FV^3)]^\pospol \\
          \boolposapp(\clause{C}_3) \colon &
             \clause{C}_6 \colon [\sk^{\mathrm{inv}} \; (\sk \; \FV^4) \; \FV^5]^\negpol \lor
                                  [\FV^4 \; \FV^5]^\pospol  \\
          \primsubstapp(\clause{C}_6, \approxbinding{o\iota}{\neg}), \cnfapp \colon &
             \clause{C}_7 \colon [\sk^{\mathrm{inv}} \; (\sk \; (\lambda Z_\iota.\, \neg (\FV^6 \; Z))) \; \FV^5]^\negpol \lor
                                  [\FV^6 \; \FV^5]^\negpol\\
          \factorapp(\clause{C}_7),\uniapp, \cnfapp \colon &
            \clause{C}_8 \colon [\sk^{\mathrm{inv}} \;
                                  \big(\sk \; \lambda Z_\iota.\, \neg (\sk^{\mathrm{inv}} \; Z \; Z)\big) \;
                                  \big(\sk \; \lambda Z_\iota.\, \neg (\sk^{\mathrm{inv}} \; Z \; Z)\big)]^\negpol \\
          \paramodapp(\clause{C}_5, \clause{C}_8),\uniapp,\scriptsize\cnfapp \colon &
            \clause{C}_9 \colon [\sk^{\mathrm{inv}} \;
                                  \big(\sk \; \lambda Z_\iota.\, \neg (\sk^{\mathrm{inv}} \; Z \; Z)\big) \;
                                  \big(\sk \; \lambda Z_\iota.\, \neg (\sk^{\mathrm{inv}} \; Z \; Z)\big)]^\pospol \\
          \paramodapp(\clause{C}_9, \clause{C}_8), \uniapp \colon &
           \emptyclause
        \end{array}
        $}\\
        
  \noindent The introduced Skolem symbol $\sk$ is of type $\iota(o\iota)$
  and its (assumed) left inverse, denoted $\sk^{\textrm{inv}}$ of type $o\iota\iota$, is inferred by
  $(\mathrm{INJ})$ based on the injectivity specification given by clause $\clause{C}_0$.
  The inferred property of $\sk^{\textrm{inv}}$ is given by $\clause{C}_2$.
  The injective Cantor Theorem is part of the TPTP library as problem 
  \texttt{SYO037\textasciicircum1} and could
  not be solved
  by any existing HO ATP system before.
\end{example}


\section{System Architecture and Implementation}\label{sec:impl}
As mentioned before, the main goal of the Leo-III prover is to achieve
effective automation of reasoning in HOL, and, in particular, to address the
shortcomings of resolution-based approaches when handling equality. To
that end, the implementation of Leo-III uses the complete EP calculus presented in \S\ref{sec:paramod}
as a starting point, and furthermore implements the rules from \S\ref{ssec:extcalculus}.
Although EP is still unordered
and Leo-III therefore generally suffers from the same drawbacks as experienced in first-order
paramodulation, including state space explosions and a prolific proof search, the idea is to use EP anyway
as a basis for Leo-III and to pragmatically tackle the problems with additional calculus rules
(cf. \S\ref{ssec:extcalculus}),
and optimizations and heuristics on the implementation level.
As a further
technical adjustment, the term representation data structures of Leo-III do not assume primitive
equality to be the only primitive logical connective. While this is handy from a theoretical point of view,
an explicit representation of further primitive logical connectives is beneficial from a practical side.

An overview of Leo-III's top level architecture is displayed in Fig.~\ref{fig:architecture}.
After parsing the problem statement, a symbol-based relevance filter adopted from Meng and Paulson~\cite{meng2009lightweight}
is employed for premise selection. The input formulas that pass the relevance filter are translated
into polymorphically typed $\lambda$-terms (Interpreter) and are then passed \hl{to the \textbf{Proof Search} component.}
In this component, the main top-level proof search algorithm is implemented.
Subsequent to initial pre-processing, this algorithm repeatedly invokes procedures from a dedicated component,
\hl{denoted \textbf{Control}, that} acts as a facade to the concrete calculus rules, and moreover manages, selects
and applies different heuristics that may restrict or guide the application of the calculus rules.
These decisions are based on current features of the proof search progress
and user-provided parameter values; such information is bundled in a \hl{dedicated \textbf{State} object.}
If the proof search is successful, the Proof Search component may output a proof object
that is constructed by the Control module on request. Indexing data structures are employed for speeding up
frequently used procedures.

\begin{figure}[tb]
\centering
\includegraphics[width=\textwidth]{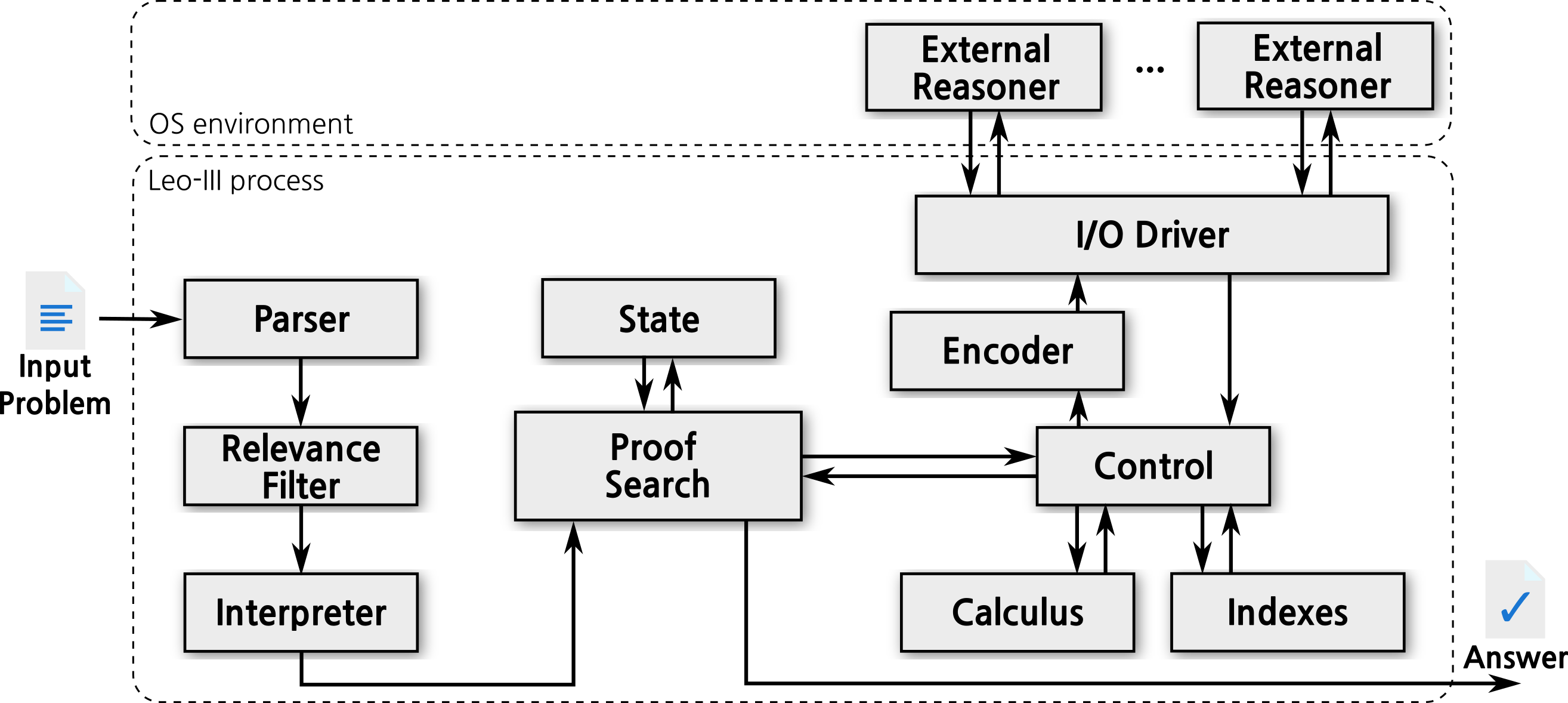}
\caption{Schematic diagram of Leo-III's architecture. The arrows indicate directed information flow. The external reasoners are executed asynchronously (non-blocking) as dedicated processes of the operating system.\label{fig:architecture}}
\end{figure}

Note that the proof search procedure itself does not have direct access \hl{to the \textbf{Calculus} module} in Fig.~\ref{fig:architecture}:
Leo-III implements a layered multi-tier architecture in which lower tiers (e.g.\ the Calculus component)
are only accessed through a higher tier (e.g.\ the Control). This allows for a modular and more flexible
structure in which single components can be replaced or enriched without major changes to others. It furthermore
improves maintainability as individual components implement fewer functionality (separation of concerns).
Following this approach, the Calculus component merely implements the inference rules of Leo-III's
calculus but it does not specify when to apply them, nor does it provide functionality to 
decide whether individual calculus rules should be applied in a given situation (e.g.\, with respect to some heuristics).
The functions provided by this module are low-level; invariantly, there are approximately as many
functions in the Calculus module as there are inference rules.
The Control component, in contrast, bundles certain inference rule applications with simplification routines,
indexing data structure updates, and heuristic decision procedures in single high-level procedures.
These procedures are then invoked by the Proof Search which passes its current search state
as an argument to the calls. The State component is then updated accordingly by
the Proof Search using the results of such function calls. The Proof Search
module implements a saturation loop that is discussed further below.

Leo-III makes use of external (mostly first-order) ATP systems for discharging proof obligations.
If any external reasoning system finds the submitted
proof obligation to be unsatisfiable, the original HOL problem is unsatisfiable as well
and a proof for the original conjecture is found.
Invocation, translation and utilization of the external results are also bundled by the Control module, 
cf.\ further below for details.

\subsection{Proof search\label{ssec:proofsearch}}
The overall proof search procedure of Leo-III consists of three consecutive phases:
preprocessing, saturation and proof reconstruction.

During preprocessing, the input formulas are transformed into a fully \sloppy
Skolemized $\beta\eta$-normal clausal normal form. In addition, 
methods including definition expansion, simplification, miniscoping,
replacement of defined equalities, and clause
renaming~\cite{DBLP:conf/cade/WisniewskiSKB16} are applied, cf. Steen's
thesis for details~\cite{steen2018}.

Saturation is organized as a sequential procedure
that iteratively saturates the set of input clauses with respect to EP (and its extensions)
until the empty clause is derived.
The clausal search space is structured
using two sets $U$ and $P$ of {unprocessed}
clauses and {processed} clauses, respectively.
Initially, $P$ is empty and $U$ contains all clauses generated from the input 
problem. 
Intuitively, the algorithm iteratively selects an unprocessed clause $g$
(the {given clause}) from
$U$. If $g$ is the empty clause, the initial clause set is shown to be
inconsistent and the algorithm terminates. If $g$ is not the empty clause,
all inferences involving $g$ and (possibly) clauses in $P$ are generated
and inserted into $U$. The resulting invariant is that all inferences between
clauses in $P$ have already been performed.
Since in most cases the number of clauses that can be generated during
proof search is infinite, the saturation process is limited artificially
using time resource bounds that can be configured by the user.

Leo-III employs a variant of the DISCOUNT~\cite{Denzinger1997} loop
that has its intellectual
roots in the E prover~\cite{Schulz:2002:EBT:1218615.1218621}.
Nevertheless, some modifications are necessary to address the specific
requirements of reasoning in HOL. 
Firstly, since formulas can occur within subterm positions and, in particular,
within proper equalities, many of the generating and modifying inferences may 
produce non-CNF clauses albeit having proper clauses as premises. This implies
that, during a proof loop iteration, potentially every clause needs to be renormalized.
Secondly, since higher-order unification is undecidable, unification procedures
cannot be used as an eager inference filtering mechanism (e.g., for paramodulation and
factoring) nor can they be integrated as an isolated procedure on the meta-level 
as done in first-order procedures.
As opposed to the first-order case, clauses that have unsolvable unification constraints
are not discarded but nevertheless inserted into the search space.
This is necessary in order to retain completeness. 

If the empty clause was inferred during saturation and the user requested a
proof output, a proof object is generated using backwards traversal of the
respective search subspace. Proofs in Leo-III are presented as
TSTP refutations~\cite{Sut07-CSR}, cf. \S\ref{ssec:proofs} for details.

\subsection{Polymorphic Reasoning}
Proof assistants such as Isabelle/HOL~\cite{DBLP:books/sp/NipkowPW02} and
Coq~\cite{DBLP:series/txtcs/BertotC04} are based on type systems that extend simple types
with, e.g., polymorphism, type classes, dependent types and further type concepts.
Such expressive type systems allow structuring knowledge
in terms of reusability and are of major importance in practice.

Leo-III supports reasoning in first-order and higher-order logic
with rank-1 polymorphism.
The support for polymorphism has been strongly
influenced by the recent development of the TH1 format for
representing problems in rank-1 polymorphic HOL~\cite{DBLP:conf/cade/KaliszykSR16},
extending the standard THF syntax~\cite{DBLP:journals/jfrea/SutcliffeB10} for HOL.
The extension of Leo-III to polymorphic reasoning does not require modifications of the general proof
search process as presented further above. Also, the data structures of Leo-III are already expressive enough to
represent polymorphic formulas, cf. technical details in earlier work~\cite{DBLP:conf/lpar/SteenWB17}.

Central to the polymorphic variant of Leo-III's calculus is the notion of {type unification}.
Type unification between two types $\tau$ and $\nu$ yields a substitution $\sigma$ such that
$\tau\sigma \equiv \nu\sigma$, if such a substitution exists. The most general type unifier
is then defined analogously to term unifiers.
Since unification on rank-1 polymorphic types is essentially a first-order
unification problem, it is decidable and unitary, i.e., it yields
a unique most general unifier if one exists.
Intuitively, whenever a calculus rule of EP requires two premises to have the same
type,
it then suffices in the polymorphic extension of EP to require that the types are unifiable.
For a concrete inference, the type unification is then applied first to the clauses, followed
by the standard inference rule itself.

Additionally, Skolemization needs to be adapted to account for
free type variables in the scope of existentially quantified variables.
As a consequence, Skolem constants that are introduced, e.g., during clausification
are polymorphically typed symbols $\mathrm{sk}$ that are applied to the free type variables $\overline{\alpha^i}$
followed by the free term variables $\overline{X^i}$, yielding the final Skolem term
$(\mathrm{sk} \; \overline{\alpha^i} \; \overline{X^i})$, where $\mathrm{sk}$ is the fresh
Skolem constant. A similar construction is used for general bindings that
are employed by primitive substitution or projection bindings during unification.
A related approach is employed by Wand in the extension of the first-order ATP system SPASS
to polymorphic types~\cite{DBLP:phd/hal/Wand17}.

A full investigation of the formal properties of these calculus extensions to polymorphic
HOL is further work.

\subsection{External Cooperation}
Leo-III's saturation procedure periodically requests the invocation of 
external reasoning systems at the Control module for discharging proof obligations
that originate from its current search space. Upon request the Control module
checks, among other things, whether the set of processed clauses $P$ changed significantly since the last
request. If this is the case and the request is granted, the search space is enqueued for
submission to external provers. If there are no external calls awaiting to be executed,
the request is automatically granted. This process is visualized in Fig.~\ref{fig:externalcoop}. 
The invocation of an external ATP system is executed asynchronously 
(non-blocking), hence the internal proof search continues while awaiting the result of the external system.
Furthermore, as a consequence of the non-blocking nature of external cooperation, multiple
external reasoning systems (or multiple instances of the same) may be employed in parallel.
To that end, a \hl{dedicated \textbf{I/O driver} is implemented} that manages the asynchronous 
calls to external processes and collects the incoming results.

\begin{figure}[tb]
\centering
\includegraphics[width=\textwidth]{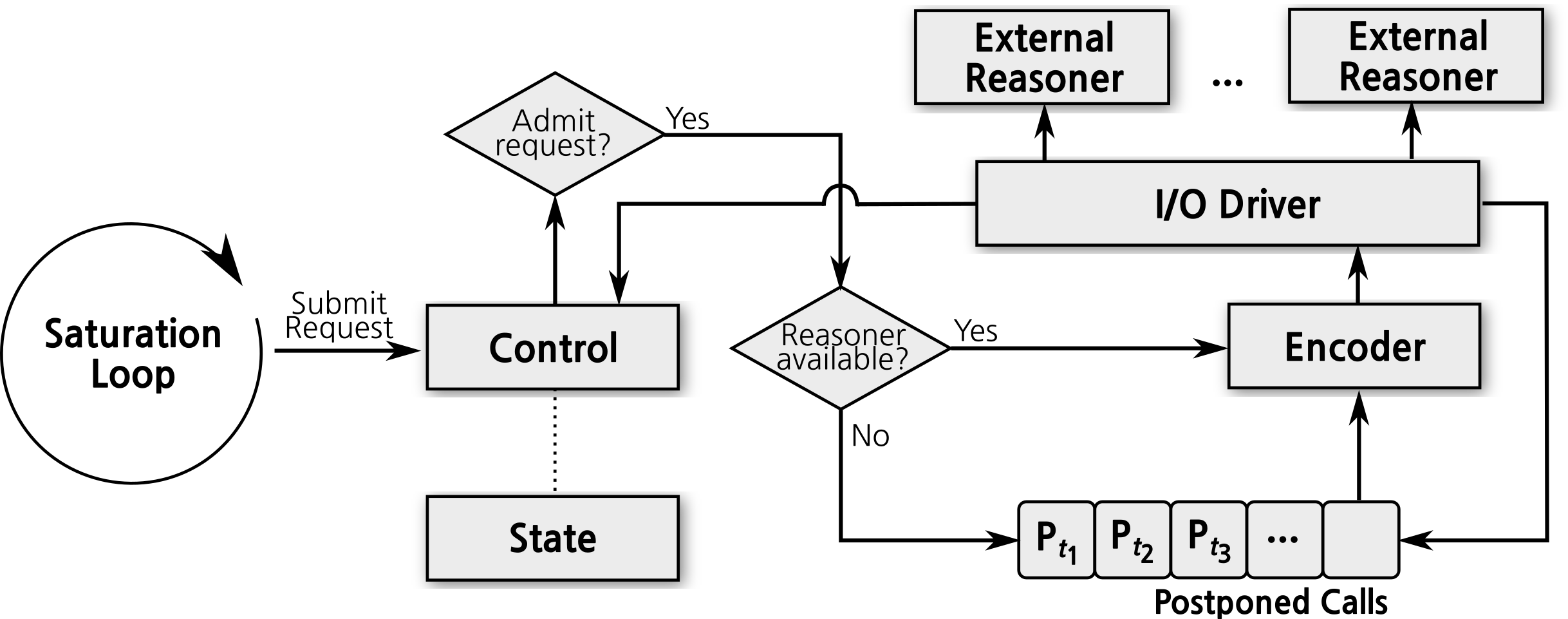}
\caption{Invocation of external reasoning systems during proof search. The solid arrows denote data flow
  through the respective modules of Leo-III. A dotted line indicates indirect use of auxiliary
  information. Postponed calls are selected by the I/O driver
  after termination of outstanding external calls. \label{fig:externalcoop}}
\end{figure}

The use of different external reasoning systems is also challenging from
a practical perspective: Different ATP systems support different logics
and furthermore different syntactical fragments of these logics. This is addressed in Leo-III
using an encoding \hl{module (cf. \textbf{Encoder} in Fig.~\ref{fig:architecture} resp.\ Fig.~\ref{fig:externalcoop})}
that translates the polymorphically typed higher-order clauses to monomorphic higher-order formulas,
or to polymorphic or monomorphic typed first-order clauses. It also removes unsupported
language features and replaces them with equivalent formulations. 
Fig.~\ref{fig:translation} displays the translation pipeline of Leo-III
for connecting to external ATP systems.
The Control module of Leo-III will automatically
select the encoding target to be the most expressive logical language that is still supported
by the external system~\cite{DBLP:conf/lpar/SteenWSB17}.
The translation process combines heuristic
monomorphization~\cite{boehmephd,DBLP:journals/corr/BlanchetteB0S16} steps
with standard encodings of higher-order language features~\cite{DBLP:journals/jar/MengP08} in first-order logic.
For some configurations there are multiple possible approaches (e.g., either monomorphize
from TH1 to TH0 and then encode to TF1, or encode directly to TF1), in these cases a default is 
fixed but the user may chose otherwise via command-line parameters.

While LEO-II relied on cooperation with untyped first-order provers, Leo-III
aims at exploiting the relatively young support of simple types in first-order ATP systems.
As a consequence, the translation of higher-order proof obligations 
does not require the encoding of types as terms, e.g.,
by type guards or type tags~\cite{DBLP:conf/cade/CouchotL07,DBLP:journals/corr/BlanchetteB0S16}.
This approach reduces clutter and hence promises more effective cooperation.
Cooperation within Leo-III is by no means limited to
first-order provers. Various different systems, including first-order and higher-order ATP systems and
model finders, can in fact be used simultaneously, provided that they comply
with some common TPTP language standard. 
Supported TPTP languages for external cooperation include the TPTP dialects TF0~\cite{DBLP:conf/lpar/SutcliffeSCB12}, TF1~\cite{DBLP:conf/cade/BlanchetteP13},
TH0~\cite{DBLP:journals/jfrea/SutcliffeB10} and TH1~\cite{DBLP:conf/cade/KaliszykSR16}.

\begin{figure}[tb]
\centering
\includegraphics[width=\textwidth]{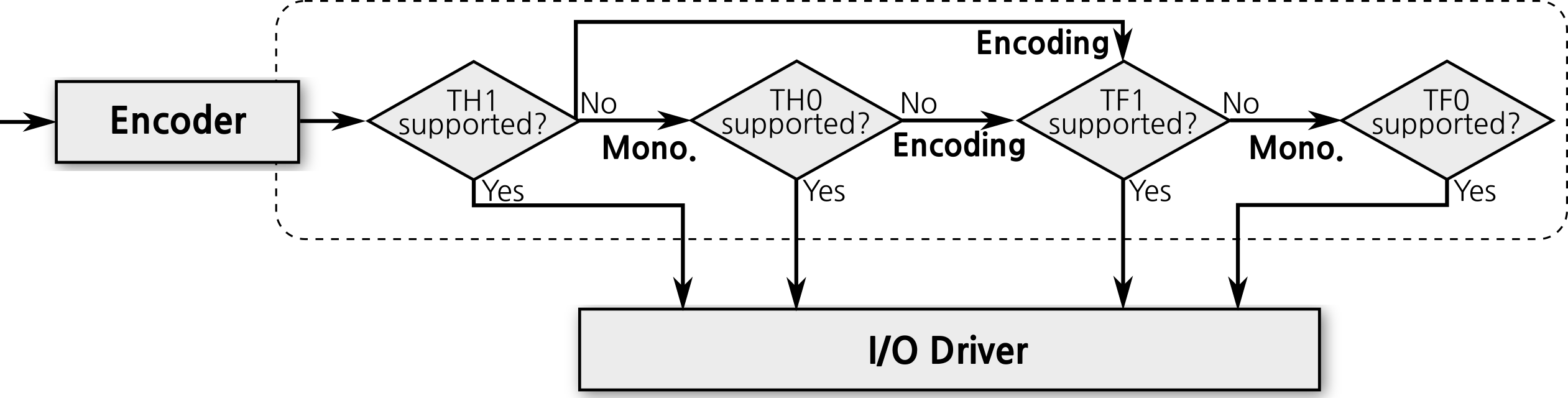}
\caption{Translation process in the encoder module of Leo-III. Depending on the supported logic fragments of the respective external reasoner, the clause set is translated to different logic formalism: Polymorphic HOL (TH1), monomorphic HOL (TH0), polymorphic first-order logic (TF1) or monomorphic (many-sorted) first-order logic (TF0). \label{fig:translation}}
\end{figure}

\subsection{Input and Output\label{ssec:proofs}}
Leo-III accepts all common TPTP dialects~\cite{DBLP:journals/jar/Sutcliffe17},
including untyped clause normal form (CNF), untyped and typed first-order logic
(FOF and TFF, respectively) and, as primary input format,
monomorphic higher-order logic (THF)~\cite{DBLP:journals/jfrea/SutcliffeB10}.
Additionally, Leo-III is one of the first higher-order ATP systems to support
reasoning in rank-1 polymorphic variants of these logics using the
TF1~\cite{DBLP:conf/cade/BlanchetteP13} and TH1~\cite{DBLP:conf/cade/KaliszykSR16} languages.

Leo-III rigorously implements the machine-readable
TSTP result standard~\cite{Sut07-CSR}
and hence outputs appropriate SZS ontology values~\cite{sutcliffe2008szs}.
The use of the TSTP output format allows for simple means of
communication and exchange of reasoning results between different reasoning tools
and, consequently, eases the employment of Leo-III within external tools.
Novel to the list of supported SZS result values for the Leo prover family is 
\texttt{ContradictoryAxioms}~\cite{sutcliffe2008szs}, which is reported if the input axioms 
were found to be inconsistent during the proof run
(i.e., if the empty clause could be derived without using the conjecture even once).
Using this simple approach, Leo-III identified 15 problems from the TPTP library to be inconsistent without any special setup.

Additional to the above described SZS result value,
Leo-III produces machine readable proof certificates if a proof was found
and such a certificate has been requested.
Proof certificates are an ASCII encoded, linearized, directed acyclic graph (DAG) of inferences
that refutes the negated input conjecture by ultimately generating the empty clause.
The root sources of the inference DAG are the given conjecture (if any) and
all axioms that have been used in the refutation.
The proof output records all intermediate inferences.
The representation again follows the TSTP format and records the inferences
using annotated \texttt{THF} formulas.
Due to the fine granularity of Leo-III proofs, it is often possible to verify them
step-by-step using external tools such as GDV~\cite{DBLP:journals/ijait/Sutcliffe06}.
A detailed description of Leo-III's proof output format and the information contained
therein can be found in Steen's PhD thesis~\cite[\S 4.5]{steen2018}.
An example of such a proof output is displayed in Fig.~\ref{fig:output}.

\begin{figure}[tb]
\begin{lstlisting}[basicstyle=\ttfamily\scriptsize,frame=single,keywordstyle=\bfseries,morekeywords={thf,inference,$false,$true,bind,file,SZS,Theorem,CNFRefutation}]
% SZS status Theorem for sur_cantor_th1.p
% SZS output start CNFRefutation for sur_cantor_th1.p
thf(skt1_type, type, skt1: $tType).
thf(sk1_type, type, sk1: (skt1 > (skt1 > $o))).
thf(sk2_type, type, sk2: ((skt1 > $o) > skt1)).
thf(1, conjecture, ! [T: $tType]: (
                              ~ ( ?[F:T > (T > $o)]: (
                                    ![Y:T > $o]: ?[X:T]: ((F @ X) = Y) ) )),
    file('sur_cantor_th1.p',sur_cantor) ).
thf(2, negated_conjecture, ~ ! [T:$tType]: (
                              ~ ( ?[F:T > (T > $o)]: (
                                    ![Y:T > $o]: ?[X:T]: ((F @ X) = Y) ) )),
    inference(neg_conjecture,[status(cth)],[1]) ).
thf(4,plain,! [A:skt1 > $o]: (sk1 @ (sk2 @ A) = A),
    inference(cnf,[status(esa)],[2]) ).
thf(6,plain,! [B:skt1, A:skt1 > $o]: ((sk1 @ (sk2 @ A) @ B) = (A @ B)),
    inference(func_ext,[status(esa)],[4])).
thf(8,plain,! [B:skt1, A:skt1 > $o]: ((sk1 @ (sk2 @ A) @ B) | ~ (A @ B)),
    inference(bool_ext,[status(thm)],[6])).
thf(272,plain,! [B:skt1, A:skt1 > $o]: ( sk1 @ (sk2 @ A) @ B) | 
                             ((A @ B) != ~ (sk1 @ (sk2 @ A) @ B)) | ~$true),
    inference(eqfactor_ordered,[status(thm)],[8])).
thf(294,plain,sk1 @ (sk2 @ (^ [A:skt1]: ~ (sk1 @ A @ A)))
                   @ (sk2 @ (^ [A:skt1]: ~ (sk1 @ A @ A))),
    inference(pre_uni,[status(thm)],[272:[
              bind(A, $thf(^ [C:skt1]: ~ (sk1 @ C @ C))),
              bind(B, $thf(sk2 @ (^ [C:skt1]: ~ (sk1 @ C @ C))))]])).
thf(7,plain,! [B:skt1, A:skt1 > $o]: (~ (sk1 @ (sk2 @ A) @ B)) | (A @ B)),
    inference(bool_ext,[status(thm)],[6])).
thf(17,plain,! [B:skt1, A:skt1 > $o]: ( (~ (sk1 @ (sk2 @ A) @ B)) | 
                           ((A @ B) != (~ (sk1 @ (sk2 @ A) @ B))) | ~$true),
    inference(eqfactor_ordered,[status(thm)],[7])).
thf(33,plain,~ (sk1 @ (sk2 @ (^ [A:skt1]: ~ (sk1 @ A @ A)))
                     @ (sk2 @ (^ [A:skt1]: ~ (sk1 @ A @ A)))),
    inference(pre_uni,[status(thm)],[17:[
              bind(A, $thf(^ [C:skt1]: ~ (sk1 @ C @ C))),
              bind(B, $thf(sk2 @ (^ [C:skt1]: ~ (sk1 @ C @ C))))]])).
thf(320,plain,$false,inference(rewrite,[status(thm)],[294,33])).
% SZS output end CNFRefutation for sur_cantor_th1.p
\end{lstlisting}
\caption{Proof output of Leo-III for the polymorphic variant (TH1 syntax) of the surjective variant of Cantor's theorem.\label{fig:output}}
\end{figure}

\subsection{Data Structures}

Leo-III implements a combination of term representation techniques;
term data structures are provided that 
admit expressive typing, efficient basic term operations and reasonable
memory consumption~\cite{DBLP:conf/lpar/BenzmullerSW17}.
Leo-III employs
a so-called spine notation~\cite{DBLP:journals/logcom/CervesatoP03},
which imitates first-order-like terms
in a higher-order setting. Here, terms are either type abstractions,
term abstractions or applications
of the form $f \cdot (s_1;s_2;\ldots)$, where the head $f$ 
is either a constant symbol, a bound variable or a complex term, and
the spine $(s_1;s_2;\ldots)$ is a linear list of arguments that are, again, spine terms.
Note that if a term is $\beta$-normal, $f$ cannot be a complex term. This observation
leads to an internal distinction between $\beta$-normal and (possibly) non-$\beta$-normal
spine terms. The first kind has an optimized representation, where the head is
only associated with an integer representing a constant symbol or variable.

Additionally,
the term representation incorporates explicit substitutions~\cite{DBLP:journals/jfp/AbadiCCL91}.
In a setting of explicit substitutions,
substitutions are part of the term language and can thus be
postponed and composed before being applied to the term. This technique admits
more efficient $\beta$-normalization and substitution operations as terms
are traversed only once, regardless of the number of substitutions applied.

Furthermore, Leo-III implements a locally nameless
representation using de Bruijn indices~\cite{Bruijn72}. 
In the setting of polymorphism~\cite{DBLP:conf/lpar/SteenWB17}, types may also contain
variables. Consequently, the nameless representation of variables is extended
to type variables~\cite{KRTU99}. The definition of de Bruijn indices for type
 variables is analogous to the one for term variables. In fact, since only
rank-1 polymorphism is used, type indices are much easier to manage than
term indices. This is due to the fact that there are no type quantifications
except for those on top level.
One of the most important advantages
of nameless representations over representations with explicit variable names
is that $\alpha$-equivalence is reduced to syntactical equality, i.e.,
two terms are $\alpha$-equivalent if and only if their nameless representations are equal.

Terms are perfectly shared within
Leo-III, meaning that each term is only constructed once and then reused
between different occurrences. This reduces memory
consumption in large knowledge bases and it allows constant time
term comparison for syntactic equality using the term's pointer to its
unique physical representation. For fast basic term retrieval operations 
(such as access of a head symbol, subterm occurrences, etc.)
terms are kept in $\beta$-normal $\eta$-long form.

A collection of basic data structures and algorithms for the implementation of
higher-order reasoning systems has been isolated from the
implementation of Leo-III into a dedicated framework called
\textsc{LeoPARD}~\cite{DBLP:conf/mkm/WisniewskiSB15}, which is freely available
at GitHub.\footnote{
  \underline{Leo}’s \underline{P}arallel \underline{Ar}chitecture and \underline{D}atastructures (\textsc{LeoPARD}) can be found at \url{https://github.com/leoprover/LeoPARD}.
}
This framework
provides many stand-alone components, including a term data
structure for polymorphic $\lambda$-terms, unification and subsumption
procedures, parsers for all TPTP languages, and further utility
procedures and pretty printers for TSTP compatible proof 
representations.


\section{Reasoning in Non-Classical Logics}\label{sec:ncl}
\newcommand{\nec}{\Box} \newcommand{\pos}{\Diamond}

Computer-assisted reasoning in non-classical logics (NCL) is of increasing
relevance for applications in artificial intelligence, computer science,
mathematics and philosophy. 
However, with a few exceptions, most of the available reasoning systems
focus on classical logics only, including common contemporary first-order and
higher-order theorem proving systems. In particular for quantified
NCLs there are only very few systems available to
date.

As an alternative to the development of specialized theorem proving systems,
usually one for each targeted NCL, a shallow semantical embedding (SSE) approach allows
for a simple adaptation of existing higher-order reasoning systems 
to a broad variety of expressive logics~\cite{J41}. 
In the SSE approach, the non-classical target logic is shallowly embedded
in HOL by providing a direct encoding of its semantics, typically
a set theoretic or relational semantics, within the term language of HOL.
As a consequence, showing validity in the target logic is reduced to higher-order
reasoning and HOL ATP systems can be applied for this task.
Note that this technique, in principle, also allows off-the-shelf automation
even for quantified NCLs as quantification and binding mechanisms of the
HOL meta logic can be utilized.
This is an interesting option in many application areas, e.g., in 
ethical and legal reasoning, as the respective communities
do not yet agree on which logical system should actually be preferred.
The resource-intensive implementation of dedicated new provers
for each potential system is not an adequate option for
rapid prototyping of prospective candidate logics and can be avoided using SSEs.

Leo-III is addressing this gap. In addition to its HOL reasoning capabilities,
it is the first system that natively supports reasoning in a wide range of
normal higher-order modal logics (HOMLs)~\cite{DBLP:conf/lpar/GleissnerSB17}.
To achieve this, Leo-III internally implements the SSE approach
for quantified modal logics based on their Kripke-style
semantics~\cite{blackburn2006handbook,DBLP:journals/igpl/BenzmullerP10}.

Quantified modal logics are associated with many different notions of semantics~\cite{blackburn2006handbook}.
Differences may, e.g., occur in the interaction between quantifiers and
the modal operators, as expressed by the Barcan formulas~\cite{DBLP:journals/jsyml/Barcan46},
or regarding the interpretation of constant symbols as rigid or non-rigid.
Hence, there are various subtle but meaningful variations in multiple individual facets of which many
combinations yield a distinct modal logic.
Since many of those variations have their particular applications, there is no
reasonably small subset of generally preferred modal logics to which a theorem proving system
should be restricted.
This, of course, poses a major practical challenge.
Leo-III, therefore, supports all quantified Kripke-complete normal modal logics~\cite{DBLP:conf/lpar/GleissnerSB17}.
In contrast, other ATP systems for (first-order) quantified modal logics
such as MleanCoP~\cite{DBLP:conf/cade/Otten14} and MSPASS~\cite{DBLP:conf/tableaux/HustadtS00}
only support a comparably small subset of all possible semantic variants. 

\begin{figure}[t]
\centering
\begin{lstlisting}[basicstyle=\scriptsize\ttfamily,frame=single,keywordstyle=\bfseries,morekeywords={thf,$constants,$quantification,$consequence,$modalities}]
thf(s5_spec, logic, ($modal := [
      $constants := $rigid, $quantification := $constant,
      $consequence := $global, $modalities := $modal_system_S5 ])).
thf(becker,conjecture,( ! [P:$i>$o,F:$i>$i, X:$i]: (? [G:$i>$i]:
      (($dia @ ($box @ (P @ (F @ X)))) => ($box @ (P @ (G @ X))))))). 
\end{lstlisting} 
\caption{A corollary of Becker's postulate formulated in modal THF, representing the 
formula $\forall P_{o\iota}.\,\forall F_{\iota\iota}.\,\forall
X_\iota.\, \exists G_{\iota\iota}.\, (\pos \nec P(F(X)) \Rightarrow \nec
P(G(X)))$. The first statement specifies the modal logic to be logic S5 with constant domain quantification, rigid constant symbols and a global consequence relation.\label{fig:beckerex}}
\end{figure}

Unlike in classical logic, a problem statement comprised only of axioms and a conjecture
to prove
does not yet fully specify a reasoning task in quantified modal logic.
It is necessary to also explicitly state the intended semantical details in which the problem
is to be attacked.
This is realized by including a meta-logical specification entry in the header of the modal logic
problem file in form of a TPTP THF formula of role \verb|logic|.
This formula then specifies respective details
for each relevant semantic dimension, cf.~\cite{DBLP:conf/ruleml/GleissnerS18} for more details
on the specification syntax. An example is displayed in Fig.~\ref{fig:beckerex}.
The identifiers \verb|$constants|, \verb|$quantification|
and \verb|$consequence| in the given case specify that constant symbols are rigid,
that the quantification semantics is constant domain, and that the consequence relation is global, respectively,
and \verb|$modalities| specifies the properties of the modal connectives by means of
fixed modal logic system names, such as S5 in the given case, or, alternatively, by listing individual names of modal axiom schemes.
This logic specification approach was developed in earlier work~\cite{DBLP:conf/cade/WisniewskiSB16}
and subsequently improved and enhanced to a work-in-progress TPTP language extension proposal.\footnote{
   See \url{http://tptp.org/TPTP/Proposals/LogicSpecification.html}.
  }

When being invoked on a modal logic problem file as displayed in Fig.~\ref{fig:beckerex}, 
Leo-III parses and analyses the logic specification part,
automatically selects and unfolds the corresponding definitions
of the SSE approach, adds appropriate axioms and then starts reasoning in (meta logic)
HOL. This process is visualized in Fig.~\ref{fig:embedding}.
Subsequently, Leo-III returns SZS 
compliant result information and, if successful, also a proof object in TSTP format.
Leo-III's proof output for the example from Fig.~\ref{fig:beckerex} is displayed in
Appendix~\ref{appendix:beckerproof}; it shows the relevant SSE definitions that have been
automatically generated by Leo-III according to the given logic specification, and this
file can be verified by GDV~\cite{DBLP:journals/jar/Sutcliffe17}.
Previous experiments~\cite{DBLP:conf/lpar/GleissnerSB17,DBLP:conf/lpar/BenzmullerR13} have shown
that the SSE approach offers an
effective automation of embedded non-classical logics for the case of quantified modal logics.
  
As of version 1.2,
Leo-III supports, but is not limited to, first-order and higher-order extensions of the well known modal
logic cube~\cite{blackburn2006handbook}. When taking the different parameter combinations into account
this amounts to more than 120 supported HOMLs.
The exact number of supported logics is in fact
much higher, since Leo-III also supports multi-modal logics
with independent modal system specification for each modality.
Also, user-defined combinations of rigid and non-rigid constants 
and different quantification semantics per type domain are possible.
In addition to modal logic reasoning, Leo-III also integrates SSEs of deontic
logics~\cite{DBLP:conf/deon/BenzmullerFP18}.

\begin{figure}[t]
    \centering
    \includegraphics[width=.9\textwidth]{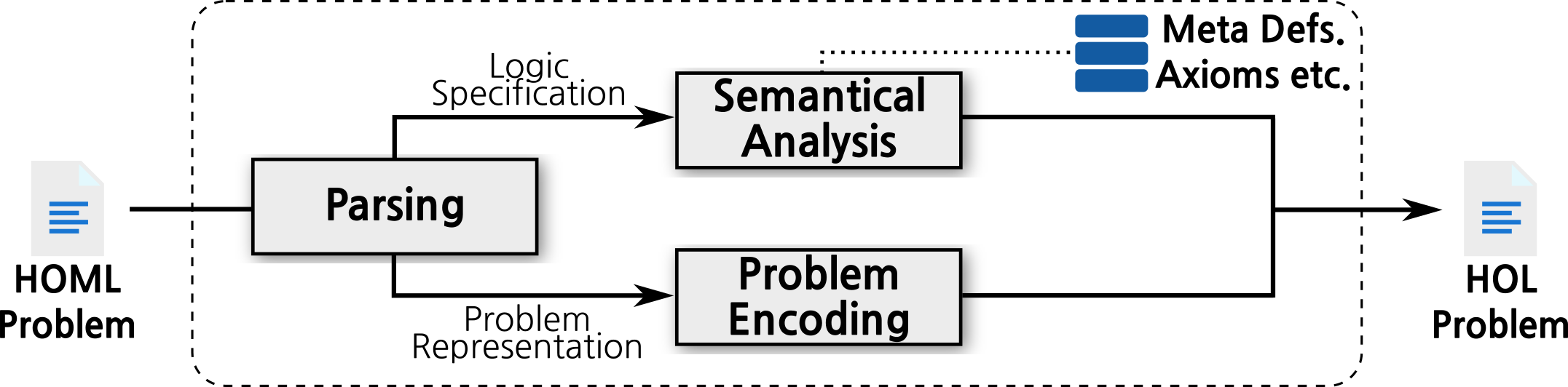}
    \caption{Schematic structure of the embedding preprocessing procedure in Leo-III.\label{fig:embedding}}
  \end{figure}


\section{Evaluation}\label{sec:eval}
In order to quantify the performance of Leo-III, an
  evaluation based on various benchmarks was conducted, cf.~\cite{DBLP:conf/cade/SteenB18}. 
  Three benchmark data sets were used:
  \begin{itemize}
    \item \emph{TPTP TH0} (2463 problems) is the set of all monomorphic HOL (TH0)
          problems from the TPTP library v7.0.0~\cite{DBLP:journals/jar/Sutcliffe17}
          that are annotated as theorems.
          The TPTP library is a de facto standard for the evaluation of ATP systems.
    \item \emph{TPTP TH1} (442 problems) is the subset of all 666 polymorphic
          HOL (TH1) problems from TPTP v7.0.0 that are annotated
          as theorems and do not contain arithmetic.
          The problems mainly consist of HOL Light core exports and
          Sledgehammer translations of various Isabelle/HOL theories.
    \item \emph{QMLTP} (580 problems) is the subset of
          all mono-modal benchmarks from the QMLTP library
          1.1~\cite{DBLP:conf/cade/RathsO12}. The QMLTP library only contains
          propositional and first-order modal logic problems.
          Since each problem may have a different validity status for each
          semantic notion of modal logic, all problems are selected.
          The total number of tested benchmarks in this category
          thus is 580 (raw problems) $\times$ 5 (modal systems) $\times$ 3
          (quantification semantics). QMLTP assumes rigid constant symbols 
          and a local consequence relation; this is adopted here.

  \end{itemize}
 
  \noindent The evaluation measurements were taken on the StarExec cluster
  in which each compute node is a \SI{64}{\bit} Red Hat Linux (kernel 3.10.0) machine with
  \SI{2.40}{\giga\hertz} quad-core processors and a main memory of \SI{128}{\giga\byte}.
  For each problem, every prover was given a CPU time limit of \SI{240}{\second}.
  The following theorem provers were employed in one or more of the
  benchmark sets (indicated in parentheses):
  Leo-III 1.2 (TH0, TH1, QMLTP) used with E, CVC4 and iProver as external first-order ATP systems,
  Isabelle/HOL 2016~\cite{DBLP:books/sp/NipkowPW02} (TH0, TH1), Satallax 3.0~\cite{DBLP:conf/cade/Brown12} (TH0),
  Satallax 3.2 (TH0), LEO-II 1.7.0 (TH0), Zipperposition 1.1 (TH0) and
  MleanCoP 1.3~\cite{DBLP:conf/cade/Otten14}
  (QMLTP).
  
  The experimental results are discussed next; additional details on
  Leo-III's performance are presented in Steen's
  thesis~\cite{steen2018}.
  
  \paragraph{TPTP TH0.} Table~\ref{table:leo:eval} (a) displays each system's performance on the TPTP TH0 data set.
  For each system the absolute number (Abs.) and relative share (Rel.) of solutions is
  displayed. Solution here means that a system is able to establish the SZS status \texttt{Theorem}
  and also emits a proof certificate that substantiates this claim.
  All results of the system, whether successful or not, are counted and categorized
  as THM (\texttt{Theorem}), CAX (\texttt{ContradictoryAxioms}), GUP (\texttt{GaveUp}) and TMO (\texttt{TimeOut})
  for the respective SZS status of the returned
  result.\footnote{Remark on CAX: In this special case of THM
    (theorem) the given axioms are inconsistent so that
    anything follows, including the given conjecture. Hence,
    it is counted against solved problems.}
  Additionally, the average and sum of all CPU times
  and wall clock (WC) times over all solved problems is presented.

  Leo-III successfully solves 2053 of 2463 problems
  (roughly \SI{83.39}{\percent})
  from the TPTP TH0 data set. This is 735 (\SI{35.8}{\percent}) more than Zipperposition, 
  264 (\SI{12.86}{\percent}) more than LEO-II and 81 (\SI{3.95}{\percent}) more than Satallax 3.0.
  The only ATP system that solves more problems is the most recent version of Satallax
  (3.2) that successfully solves 2140 problems, which is approximately
  \SI{4.24}{\percent} more than Leo-III. 
  Isabelle currently does not emit proof certificates (hence zero solutions).
  Even if results without explicit proofs are counted,
  Leo-III would still have a slightly higher number of problems solved than Satallax 3.0 and
  Isabelle/HOL with 25 (\SI{1.22}{\percent}) and 31 (\SI{1.51}{\percent}) additional solutions, respectively.
  Leo-III, Satallax (3.2), Zipperposition and LEO-II produce 18,
  17, 15 and 3 unique solutions, respectively. Evidently, Leo-III currently
  produces more unique solutions than any other ATP system in this setting.
  %
  Leo-III
  solves twelve problems that are currently
  not solved by any other system indexed by TPTP.\footnote{
  This information is extracted from the TPTP
  problem rating information that is attached to each problem.
  The unsolved problems are \texttt{NLP004\textasciicircum7},
    \texttt{SET013\textasciicircum7}, \texttt{SEU558\textasciicircum1}, \texttt{SEU683\textasciicircum1},
    \texttt{SEV143\textasciicircum5},
    \texttt{SYO037\textasciicircum1}, \texttt{SYO062\textasciicircum4.004},
    \texttt{SYO065\textasciicircum4.001},
    \texttt{SYO066\textasciicircum4.004}, 
    \texttt{MSC007\textasciicircum1.003.004}, \texttt{SEU938\textasciicircum5} and 
    \texttt{SEV106\textasciicircum5}.
  }
  
  \begin{table}[tb]
  \centering
  \caption{Detailed result of the benchmark measurements.\label{table:leo:eval}}
  \subfloat[TPTP TH0 data set (2463 problems)]{
    \resizebox{\linewidth}{!}{
    \begin{tabular}{l|rr|rrrr|rr|rr}
    \textbf{Systems}       & \multicolumn{2}{c|}{\textbf{Solutions}}               & \multicolumn{4}{c|}{\textbf{SZS Results}}                                                                                & \multicolumn{2}{c|}{\textbf{Avg. Time} [\si{\second}]}         & \multicolumn{2}{c}{\textbf{$\Sigma$ Time} [\si{\second}]} \\
    \textbf{}              & \multicolumn{1}{l|}{\textbf{Abs.}} & \textbf{Rel.} & \multicolumn{1}{l|}{\textbf{THM}} & \multicolumn{1}{l|}{\textbf{CAX}} & \multicolumn{1}{l|}{\textbf{GUP}} & \textbf{TMO} & \multicolumn{1}{r|}{\;\;\textbf{CPU}} & \textbf{WC} &\multicolumn{1}{l|}{\textbf{CPU}}     & \textbf{WC}     \\ \hline
    \textbf{Satallax 3.2}  & 2140                               & 86.89         & 2140                              & 0                                 & 2                                 & 321          & 12.26                             & 12.31       & 26238            & 26339           \\
    \textbf{Leo-III}              & 2053                               & 83.39         & 2045                              & 8                                 & 16                                & 394          & 15.39                             & 5.61        & 31490            & 11508           \\
    \textbf{Satallax 3.0}  & 1972                               & 80.06         & 2028                              & 0                                 & 2                                 & 433          & 17.83                             & 17.89       & 36149            & 36289           \\
    \textbf{LEO-II}        & 1788                               & 72.63         & 1789                              & 0                                 & 43                                & 631          & 5.84                              & 5.96        & 10452            & 10661           \\
    \textbf{Zipperposition}     & 1318                               & 53.51         & 1318                              & 0                                 & 360                               & 785          & 2.60                               & 2.73        & 3421             & 3592            \\
    \textbf{Isabelle/HOL} & 0                                  & 0.00             & 2022                              & 0                                 & 1                                 & 440          & 46.46                             & 33.44       & 93933            & 67610          
    \end{tabular}
    }
  }
  \\
  \subfloat[TPTP TH1 data set (442 problems)]{
    \resizebox{\linewidth}{!}{
    \begin{tabular}{l|rr|rrrr|rr|rr}
    \textbf{Systems}       & \multicolumn{2}{c|}{\textbf{Solutions}}               & \multicolumn{4}{c|}{\textbf{SZS Results}}                                                                                & \multicolumn{2}{c|}{\textbf{Avg. Time} [\si{\second}]}         & \multicolumn{2}{c}{\textbf{$\Sigma$ Time} [\si{\second}]} \\
    \textbf{}              & \multicolumn{1}{l|}{\textbf{Abs.}} & \textbf{Rel.} & \multicolumn{1}{l|}{\textbf{THM}} & \multicolumn{1}{l|}{\textbf{CAX}} & \multicolumn{1}{l|}{\textbf{GUP}} & \textbf{TMO} & \multicolumn{1}{r|}{\;\;\textbf{CPU}} & \textbf{WC} & \multicolumn{1}{l|}{\textbf{CPU}}    & \textbf{WC}     \\ \hline
    \textbf{Leo-III}  & 185 &41.86& 183 & 2 & 8 & 249 & 49.18 & 24.93 & 9099&  4613\\
    \textbf{Isabelle/HOL} & 0 &0.00& 237 & 0 & 23 & 182 & 93.53 &81.44 & 22404& 19300
    \end{tabular}
    }
  }
    
    \end{table}

  Satallax, LEO-II and Zipperposition show only small differences between
  their individual CPU and WC time on average and sum. A more precise measure for a system's utilization of 
  multiple cores is the so-called {core usage}. It is given by the
  average of the ratios of used CPU time to used wall clock time over all solved problems.
  The core usage of Leo-III for the TPTP TH0 data set is $2.52$. This means that,
  on average, two to three CPU cores are used during proof search by Leo-III.
  Satallax (3.2), LEO-II and Zipperposition show a quite opposite behavior with
  core usages of $0.64$, $0.56$ and $0.47$, respectively.
  
  \paragraph{TPTP TH1.}
  Currently, there exist only a few ATP systems that are capable of reasoning
  within polymorphic HOL as specified by TPTP TH1.
  The only exceptions are
  HOL(y)Hammer
  and Isabelle/HOL
  that schedule proof tactics within HOL Light and Isabelle/HOL,
  respectively.
  Unfortunately, only Isabelle/HOL was available for experimentation
  in a reasonably recent and stable
  version.
  Table~\ref{table:leo:eval} (b) displays the measurement results for the TPTP TH1 data
  set. When disregarding proof certificates, Isabelle/HOL finds 237 theorems (\SI{53.62}{\percent})
  which is roughly \SI{28.1}{\percent} more than the number of solutions founds by Leo-III.
  Leo-III and Isabelle/HOL produce 35 and 69 unique solutions, respectively.
  
  \begin{figure}[tb] 
	  \centering
	  \begin{tikzpicture} 
		  \begin{axis}[
			  ybar=0pt,
			  xmin = 0.5,
			  xmax = 12.5,
			  ymin = 0,
			  ymax = 470,
			  axis x line* = bottom,
			  axis y line* = left,
			  ylabel= \# solved problems,
			  width= .95\textwidth,
			  height = 0.45\textwidth,
			  ymajorgrids = true,
			  extra y ticks={100,300},
			  bar width = 3mm,
			  xtick = {1,2,3,4,5,6,7,8,9,10,11,12},
			  xticklabels = {D/vary, D/cumul, D/const, T/vary, T/cumul, T/const, S4/vary, S4/cumul, S4/const, S5/vary, S5/cumul, S5/const},
			  x tick label style={rotate=45, anchor=east},
			  legend style={at={(0.03,0.97)},anchor=north west}
			  ]
			  \addplot+[blue_1, thick,draw=none,no markers] coordinates {
			    (1,159)
			    (2,178)
			    (3,203)
			    (4,209)
			    (5,233)
			    (6,263)
			    (7,246)
			    (8,280)
			    (9,316)
			    (10,340)
			    (11,457)
			    (12,457)
			  };
			  
			  \addplot+[blue_2b, thick,draw=none,no markers] coordinates { 
			    (1,185)
			    (2,206)
			    (3,223)
			    (4,224)
			    (5,251)
			    (6,271)
			    (7,289)
			    (8,350)
			    (9,366)
			    (10,360)
			    (11,436)
			    (12,436)
			  };
			  \legend{Leo-III, MleanCoP}
		  \end{axis} 
	  \end{tikzpicture}
	  \caption{Comparison of Leo-III and MleanCoP on the QMLTP data set (580 problems).\label{fig:qmltp}}
  \end{figure}
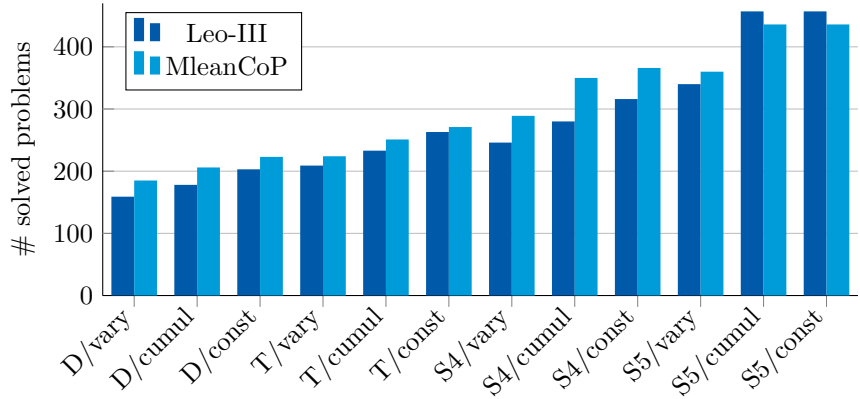
  \paragraph{QMLTP.} For each semantical setting supported by MleanCoP, which is the
  strongest first-order modal logic prover available to date~\cite{DBLP:conf/ecai/BenzmullerOR12},
  the number of theorems found 
  by both Leo-III and MleanCoP in the QMLTP data set is presented in Fig.~\ref{fig:qmltp}.
  Leo-III is fairly competitive with MleanCoP (weaker by maximal 
  \SI{14.05}{\percent}, minimal \SI{2.95}{\percent} and
  \SI{8.90}{\percent} on average) for all \textbf{D} and \textbf{T}
  variants.  For all \textbf{S4} variants, the gap between both
  systems increases (weaker by maximal \SI{20.00}{\percent}, minimal
  \SI{13.66}{\percent} and \SI{16.18}{\percent} on average). For
  \textbf{S5} variants, Leo-III is very effective (stronger by
  \SI{1.36}{\percent} on average), and it is ahead of MleanCoP for
  \textbf{S5}/{const} and \textbf{S5}/{cumul} (which coincide).  This is due to
  the encoding of the \textbf{S5} accessibility relation in Leo-III 1.2 as
  the universal relation between possible worlds as opposed to its prior 
  encoding as an equivalence relation~\cite{DBLP:conf/lpar/GleissnerSB17}.
  Note that this technically changes the possible models, but it does not change the set of
  valid theorems.
  Leo-III contributes 199 solutions to previously unsolved problems.
 

  \paragraph{On polymorphism.}
  The GRUNGE evaluation by Brown et al.~\cite{grunge} aims at comparing
  ATP systems across different supported logics. For this purpose,
  theorems from the HOL4 standard library~\cite{DBLP:conf/tphol/SlindN08}
  are translated into multiple different logical formalisms, including
  untyped first-order logic, typed first-order logic (with and without polymorphic
  types) and higher-order logic (with and without polymorphic types) using
  the different TPTP language dialects as discussed in \S\ref{ssec:proofs}.
  Of the many first-order and higher-order ATP systems that are evaluated on these data sets,
  Leo-III is one of the few to support polymorphic types.\footnote{
    HOLyHammer (HOL ATP) and Zipperposition (first-order ATP) are the only other systems
    supporting polymorphism.
  }
  This seems to be a major strength in the context of GRUNGE: Leo-III is identified
  as the most effective ATP system overall in terms of solved problems in
  any formalism, with approx. 19\% more solutions than the next best system,
  and as the best ATP system in all higher-order formalisms, with up to
  94\% more solutions than the next best higher-order system.
  Remarkably, it can be seen that over 90\% of all solved problems in the GRUNGE evaluation
  are contributed by Leo-III on the basis of the polymorphic higher-order
  data set, and the next best result in any other formalism is down by
  approx. 25\%.
  
  This suggests that reasoning in polymorphic formalisms is of particular benefit
  for applications in mathematics and, possibly, further domains. For systems without
  native support for (polymorphic) types, types are usually encoded as terms, or they
  are removed by monomorphization. This increases the complexity
  of the problem representation and decreases reasoning effectivity.
  Leo-III, on the other hand,
  handles polymorphic types natively and requires no such indirection.


\section{Conclusion and Future Work}\label{sec:conclusion}
Leo-III is an ATP system for classical HOL with Henkin semantics, and it
natively supports also various propositional and quantified non-classical logics.
This includes typed and untyped first-order logic, polymorphic
HOL, and a wide range of HOMLs, which makes Leo-III,
up to our knowledge, the most widely applicable theorem
proving system available to date.
Recent evaluations show that Leo-III is very effective (in terms of problems solved)
and that in particular its extension to polymorphic HOL is practically
relevant.

Future work includes extensions and specializations of Leo-III for selected deontic logics
and logic combinations, with the ultimate goal to support the effective automation of normative
reasoning. Additionally, stemming from the success of polymorphic reasoning in Leo-III,
a polymorphic adaption of the shallow semantical embedding approach for modal logics is planned, potentially improving
modal logic reasoning performance.


\bibliographystyle{spmpsci}      

\newpage
\appendix
\section{Leo-III Proof of Fig.~\ref{fig:beckerex}}\label{appendix:beckerproof}
\begin{lstlisting}[basicstyle=\scriptsize\ttfamily,frame=single]
% SZS status Theorem for becker.p
% SZS output start CNFRefutation for becker.p
thf(mworld_type,type,(
    mworld: $tType )).

thf(mrel_type,type,(
    mrel: mworld > mworld > $o )).

thf(meuclidean_type,type,(
    meuclidean: ( mworld > mworld > $o ) > $o )).

thf(meuclidean_def,definition,
    ( meuclidean
    = ( ^ [A: mworld > mworld > $o] :
        ! [B: mworld,C: mworld,D: mworld] :
          ( ( ( A @ B @ C )
            & ( A @ B @ D ) )
         => ( A @ C @ D ) ) ) )).

thf(mvalid_type,type,(
    mvalid: ( mworld > $o ) > $o )).

thf(mvalid_def,definition,
    ( mvalid
    = ( ^ [A: mworld > $o] : 
        ! [B: mworld] :
          ( A @ B ) ) )).

thf(mimplies_type,type,(
    mimplies: ( mworld > $o ) > ( mworld > $o ) > mworld > $o )).

thf(mimplies_def,definition,
    ( mimplies
    = ( ^ [A: mworld > $o,B: mworld > $o,C: mworld] :
          ( ( A @ C )
         => ( B @ C ) ) ) )).

thf(mdia_type,type,(
    mdia: ( mworld > $o ) > mworld > $o )).

thf(mdia_def,definition,
    ( mdia
    = ( ^ [A: mworld > $o,B: mworld] :
        ? [C: mworld] :
          ( ( mrel @ B @ C )
          & ( A @ C ) ) ) )).

thf(mbox_type,type,(
    mbox: ( mworld > $o ) > mworld > $o )).

thf(mbox_def,definition,
    ( mbox
    = ( ^ [A: mworld > $o,B: mworld] :
        ! [C: mworld] :
          ( ( mrel @ B @ C )
         => ( A @ C ) ) ) )).

thf(mexists_const__o__d_i_t__d_i_c__type,type,(
    mexists_const__o__d_i_t__d_i_c_: ( ( $i > $i ) > mworld > $o )
                                       > mworld > $o )).

thf(mexists_const__o__d_i_t__d_i_c__def,definition,
    ( mexists_const__o__d_i_t__d_i_c_
    = ( ^ [A: ( $i > $i ) > mworld > $o,B: mworld] :
        ? [C: $i > $i] :
          ( A @ C @ B ) ) )).

thf(mforall_const__o__d_i_t__o_mworld_t__d_o_c__c__type,type,(
    mforall_const__o__d_i_t__o_mworld_t__d_o_c__c_: ( ( $i > mworld > $o )
                                                      > mworld > $o )
                                                      > mworld > $o )).

thf(mforall_const__o__d_i_t__o_mworld_t__d_o_c__c__def,definition,
    ( mforall_const__o__d_i_t__o_mworld_t__d_o_c__c_
    = ( ^ [A: ( $i > mworld > $o ) > mworld > $o,B: mworld] :
        ! [C: $i > mworld > $o] :
          ( A @ C @ B ) ) )).

thf(mforall_const__o__d_i_c__type,type,(
    mforall_const__o__d_i_c_: ( $i > mworld > $o ) > mworld > $o )).

thf(mforall_const__o__d_i_c__def,definition,
    ( mforall_const__o__d_i_c_
    = ( ^ [A: $i > mworld > $o,B: mworld] :
        ! [C: $i] :
          ( A @ C @ B ) ) )).

thf(mforall_const__o__d_i_t__d_i_c__type,type,(
    mforall_const__o__d_i_t__d_i_c_: ( ( $i > $i ) > mworld > $o )
                                        > mworld > $o )).

thf(mforall_const__o__d_i_t__d_i_c__def,definition,
    ( mforall_const__o__d_i_t__d_i_c_
    = ( ^ [A: ( $i > $i ) > mworld > $o,B: mworld] :
        ! [C: $i > $i] :
          ( A @ C @ B ) ) )).

thf(sk1_type,type,(
    sk1: mworld )).

thf(sk2_type,type,(
    sk2: $i > mworld > $o )).

thf(sk3_type,type,(
    sk3: $i > $i )).

thf(sk4_type,type,(
    sk4: $i )).

thf(sk5_type,type,(
    sk5: mworld )).

thf(sk6_type,type,(
    sk6: ( $i > $i ) > mworld )).

thf(1,conjecture,
    ( mvalid
    @ ( mforall_const__o__d_i_t__o_mworld_t__d_o_c__c_
      @ ^ [A: $i > mworld > $o] :
          ( mforall_const__o__d_i_t__d_i_c_
          @ ^ [B: $i > $i] :
              ( mforall_const__o__d_i_c_
              @ ^ [C: $i] :
                  ( mexists_const__o__d_i_t__d_i_c_
                  @ ^ [D: $i > $i] :
                      ( mimplies
                        @ ( mdia @ ( mbox @ ( A @ ( B @ C ) ) ) )
                        @ ( mbox @ ( A @ ( D @ C ) ) ) ) ) ) ) ) ),
    file('becker.p',1)).

thf(2,negated_conjecture,(
    ~ ( mvalid
      @ ( mforall_const__o__d_i_t__o_mworld_t__d_o_c__c_
        @ ^ [A: $i > mworld > $o] :
            ( mforall_const__o__d_i_t__d_i_c_
            @ ^ [B: $i > $i] :
                ( mforall_const__o__d_i_c_
                @ ^ [C: $i] :
                    ( mexists_const__o__d_i_t__d_i_c_
                    @ ^ [D: $i > $i] :
                        ( mimplies
                          @ ( mdia @ ( mbox @ ( A @ ( B @ C ) ) ) )
                          @ ( mbox @ ( A @ ( D @ C ) ) ) ) ) ) ) ) ) ),
    inference(neg_conjecture,[status(cth)],[1])).

thf(5,plain,(
    ~ ! [A: mworld,B: $i > mworld > $o,C: $i > $i,D: $i] :
      ? [E: $i > $i] :
        ( ? [F: mworld] :
            ( ( mrel @ A @ F )
            & ! [G: mworld] :
                ( ( mrel @ F @ G )
               => ( B @ ( C @ D ) @ G ) ) )
       => ! [F: mworld] :
            ( ( mrel @ A @ F )
           => ( B @ ( E @ D ) @ F ) ) ) ),
    inference(defexp_and_simp_and_etaexpand,[status(thm)],[2])).

thf(6,plain,(
    ~ ! [A: mworld,B: $i > mworld > $o,C: $i > $i,D: $i] :
        ( ? [E: mworld] :
            ( ( mrel @ A @ E )
            & ! [F: mworld] :
                ( ( mrel @ E @ F )
               => ( B @ ( C @ D ) @ F ) ) )
       => ? [E: $i > $i] :
          ! [F: mworld] :
            ( ( mrel @ A @ F )
           => ( B @ ( E @ D ) @ F ) ) ) ),
    inference(miniscope,[status(thm)],[5])).

thf(10,plain,(
    mrel @ sk1 @ sk5 ),
    inference(cnf,[status(esa)],[6])).

thf(4,axiom,(
    meuclidean @ mrel ),
    file('becker.p',mrel_meuclidean)).

thf(15,plain,(
    ! [A: mworld,B: mworld,C: mworld] :
      ( ( ( mrel @ A @ B )
        & ( mrel @ A @ C ) )
     => ( mrel @ B @ C ) ) ),
    inference(defexp_and_simp_and_etaexpand,[status(thm)],[4])).

thf(16,plain,(
    ! [C: mworld,B: mworld,A: mworld] :
      ( ~ ( mrel @ A @ B )
      | ~ ( mrel @ A @ C )
      | ( mrel @ B @ C ) ) ),
    inference(cnf,[status(esa)],[15])).

thf(17,plain,(
    ! [C: mworld,B: mworld,A: mworld] :
      ( ~ ( mrel @ A @ C )
      | ( mrel @ B @ C )
      | ( ( mrel @ sk1 @ sk5 )
       != ( mrel @ A @ B ) ) ) ),
    inference(paramod_ordered,[status(thm)],[10,16])).

thf(18,plain,(
    ! [A: mworld] :
      ( ~ ( mrel @ sk1 @ A )
      | ( mrel @ sk5 @ A ) ) ),
    inference(pattern_uni,[status(thm)],
      [17:[bind(A,$thf(sk1)),bind(B,$thf(sk5))]])).

thf(40,plain,(
    ! [A: mworld] :
      ( ~ ( mrel @ sk1 @ A )
      | ( mrel @ sk5 @ A ) ) ),
    inference(simp,[status(thm)],[18])).

thf(9,plain,(
    ! [A: mworld] :
      ( ~ ( mrel @ sk5 @ A )
      | ( sk2 @ ( sk3 @ sk4 ) @ A ) ) ),
    inference(cnf,[status(esa)],[6])).

thf(7,plain,(
    ! [A: $i > $i] :
      ~ ( sk2 @ ( A @ sk4 ) @ ( sk6 @ A ) ) ),
    inference(cnf,[status(esa)],[6])).

thf(11,plain,(
    ! [A: $i > $i] :
      ~ ( sk2 @ ( A @ sk4 ) @ ( sk6 @ A ) ) ),
    inference(simp,[status(thm)],[7])).

thf(206,plain,(
    ! [B: $i > $i,A: mworld] :
      ( ~ ( mrel @ sk5 @ A )
      | ( ( sk2 @ ( sk3 @ sk4 ) @ A )
       != ( sk2 @ ( B @ sk4 ) @ ( sk6 @ B ) ) ) ) ),
    inference(paramod_ordered,[status(thm)],[9,11])).

thf(213,plain,(
    ~ ( mrel @ sk5 @ ( sk6 @ sk3 ) ) ),
    inference(pre_uni,[status(thm)],
      [206:[bind(A,$thf(sk6 @ sk3)),bind(B,$thf(sk3))]])).

thf(257,plain,(
    ! [A: mworld] :
      ( ~ ( mrel @ sk1 @ A )
      | ( ( mrel @ sk5 @ A )
       != ( mrel @ sk5 @ ( sk6 @ sk3 ) ) ) ) ),
    inference(paramod_ordered,[status(thm)],[40,213])).

thf(258,plain,(
    ~ ( mrel @ sk1 @ ( sk6 @ sk3 ) ) ),
    inference(pattern_uni,[status(thm)],[257:[bind(A,$thf(sk6 @ sk3))]])).

thf(8,plain,(
    ! [A: $i > $i] : ( mrel @ sk1 @ ( sk6 @ A ) ) ),
    inference(cnf,[status(esa)],[6])).

thf(12,plain,(
    ! [A: $i > $i] : ( mrel @ sk1 @ ( sk6 @ A ) ) ),
    inference(simp,[status(thm)],[8])).

thf(272,plain,( ~ $true ),
    inference(rewrite,[status(thm)],[258,12])).

thf(273,plain,( $false ),
    inference(simp,[status(thm)],[272])).
 
% SZS output end CNFRefutation for becker.p
\end{lstlisting}

\end{document}